\documentclass{article}

\usepackage{arxiv}

\usepackage{microtype}
\usepackage{graphicx}
\usepackage{subfigure}
\usepackage{booktabs} 

\usepackage{hyperref}

\usepackage{latexsym,amsmath,amsfonts,amssymb,amsthm} 
\usepackage{url}
\usepackage{graphicx} 
\usepackage{epstopdf}
\usepackage{subfigure}
\usepackage{graphics, color}
\usepackage{hyperref}
\usepackage{natbib}
\usepackage{stmaryrd}

\usepackage{enumitem,nicefrac}

\usepackage[absolute,overlay]{textpos}

\newtheorem{theorem}{Theorem}[section]
\newtheorem{lemma}[theorem]{Lemma}
\newtheorem{corollary}[theorem]{Corollary}

\newcommand{\cB}{{\cal B}}

\newcommand{\cA}{{\cal A}}

\newcommand{\cY}{{\cal Y}}
\newcommand{\cX}{{\cal X}}

\newcommand{\bu}{\boldsymbol{u}}
\newcommand{\bx}{\boldsymbol{x}}
\newcommand{\by}{\boldsymbol{y}}
\newcommand{\bz}{\boldsymbol{z}}

\newcommand{\bI}{\boldsymbol{I}}

\newcommand{\bX}{{\bf X}}
\newcommand{\bzero}{\bf 0}
\newcommand{\bZ}{{\bf Z}}

\newcommand{\bM}{{\bf M}}

\renewcommand{\k}{k}

\newcommand{\given}{ \, \vert \, }

\renewcommand{\vec}[1]{\boldsymbol{#1}}

\newcommand{\Algo}[1]{\textsc{#1}}

\newcommand{\prob}{\mathbb{P}}

\newcommand{\argmax}{\operatornamewithlimits{argmax}}

\newcommand\R{\mathbb{R}}   
\newcommand\N{\mathbb{N}}   

\usepackage[utf8]{inputenc} 
\usepackage[T1]{fontenc}    
\usepackage{hyperref}       
\usepackage{url}            
\usepackage{booktabs}       
\usepackage{amsfonts}       
\usepackage{nicefrac}       
\usepackage{microtype}      

\usepackage{cancel}
\usepackage{graphicx}
\usepackage{booktabs} 
\usepackage{amssymb}
\usepackage{color}
\usepackage{amsmath}
\usepackage{amsthm}
\usepackage{thmtools}
\usepackage{thm-restate}
\usepackage{algorithm}
\usepackage{algorithmic}
\usepackage{color}

\renewcommand{\P}{{\mathbf P}}

\newcommand{\X}{{\mathcal X}}

\renewcommand{\l}{{\boldsymbol l}}

\renewcommand{\k}{{\mathbf k}}

\newcommand{\p}{{\mathbf p}}
\newcommand{\q}{{\mathbf q}}
\renewcommand{\r}{{\mathbf r}}

\newcommand{\x}{{\mathbf x}}

\newcommand{\z}{{\mathbf z}}

\newcommand{\sign}{\text{sgn}}
\newcommand{\expect}{{\mathop{\mathbb{E}}}}

\newcommand{\btheta}{\boldsymbol \theta}
\newcommand{\bell}{\boldsymbol \ell}

\title{Stochastic Contextual Dueling Bandits under Linear Stochastic Transitivity Models}

\author{ Viktor Bengs \\
	Institute of Informatics, LMU Munich, Germany \\
	Viktor.bengs@lmu.de
	\And 
	Aadirupa Saha \\
	Microsoft Research, New York City \\
	aadirupa.saha@microsoft.com
	\And
	Eyke H\"ullermeier \\ 
	Institute of Informatics \& Munich Center of Machine Learning, LMU Munich, Germany \\
	eyke@lmu.de	
}

\hypersetup{
	pdftitle={Stochastic Contextual Dueling Bandits under Linear Stochastic Transitivity Models},
	pdfsubject={q-cs.LG, q-stat.ML},
	pdfauthor={Viktor Bengs, Aadirupa Saha, Eyke Hüllermeier},
	pdfkeywords={Bandits, Exploration, Exploitation, Preference Learning},
}

\begin{document}

	\maketitle
	
%
%
	
	\begin{abstract}
		We consider the regret minimization task in a dueling bandits problem with context information. In every round of the sequential decision problem, the learner makes a context-dependent selection of two choice alternatives (arms) to be compared with each other and receives feedback in the form of noisy preference information. We assume that the feedback process is determined by a linear stochastic transitivity model with contextualized utilities (CoLST), and the learner's task is to include the best arm (with highest latent context-dependent utility) in the duel. We propose a computationally efficient algorithm, \Algo{CoLSTIM}, which makes its choice based on imitating the feedback process using perturbed context-dependent utility estimates of the underlying CoLST model. If each arm is associated with a $d$-dimensional feature vector, we show that \Algo{CoLSTIM} achieves a regret of order $\tilde O( \sqrt{dT})$ after $T$ learning rounds. Additionally, we also establish the optimality of \Algo{CoLSTIM} by showing a lower bound for the weak regret that refines the existing average regret analysis. Our experiments demonstrate its superiority over state-of-art algorithms for special cases of CoLST models.
%
		%
	\end{abstract}

	\section{Introduction} \label{sec:introdution}
	
	The \emph{multi-armed bandit} (MAB) problem represents a class of online machine learning problems in which a learner can perform different actions (metaphorically referred to as ``pulling arms'') resulting in numerical rewards within a sequential decision process \citep{LaSz18}.  
	From a learning perspective, the main challenge of this problem is the \emph{exploration-exploitation dilemma}. The learner's lack of knowledge about the underlying reward mechanism associated with the actions forces it to try different actions often enough, so as to gradually understand these mechanisms better in the course of time (exploration). On the other side, the learner is inclined to choose actions deemed rewarding, as as to maximize the rewards cumulated over time (exploitation).

	Although this learning scenario has been used in many application areas, such as web advertising or medical treatments, it has been generalized in various ways to take specific aspects arising in practical problems into account. 
	In many applications, for example, additional information is available about the context in which an action is performed, for instance, the user profile in web advertising or a patient's medical record.
	In the \emph{contextual MAB setting} \citep{ChLiReSc11,valko2013finite}, an action's reward mechanism might depend on the contextual information, so that one or the other action is optimal depending on the context. 

	Another practically relevant extension of the classical MAB problem is the \emph{dueling} \citep{YuJo09} or \emph{preference-based}  \citep{bengs2021preference} bandit problem, where the learner's action is to select a pair of arms resulting in a noisy qualitative comparison between these arms (which arm generated the highest reward?) rather than selecting a single arm resulting in a numerical reward (what is the generated reward?).
	An example is A/B testing, where users are presented with two options from which the more preferred one should be selected, but no (meaningful) numerical reward can be measured for that choice. 
	Although the dueling bandits problem itself\,---\,just like the standard MAB problem\,---\,has already been the starting point for various generalizations, little attention has been paid to a contextual learning scenario in the same spirit as the contextual bandits.
	One reason might be that, in contrast to numerical reward learning scenarios, in preference-based learning scenarios it is often far from obvious to specify basic characteristics such as the optimal (pair of) arm(s) or a practically meaningful model assumption for the feedback mechanism.  

	One way to get around these issues, which we propose in this paper, is to use the so-called \emph{linear stochastic transitivity} (LST) models\footnote{Some authors refer to this model class also as Thurstone preference models. e.g., \citet{jin2020rank}.}, which enjoy great popularity in the fields of economics, behavioral science \citep{cattelan2012models} with the Bradley-Terry-Luce and the Thurstone-Mosteller model being their most well-known instantiations.
	They are especially famous for their use in rating or ranking systems like the Elo system \citep{elo1978rating} or TrueSkill \citep{herbrich2006trueskill}.
	These LST models describe the mechanism leading to the observation of the outcome of a comparison between two objects (choice alternatives, players, etc.) in a probabilistic way by assuming latent utility values of the objects and a noisy perception of these utility values.
	The probability of the outcome of the pairwise comparison, say A over B, is then equal to the probability that A's noisy utility value is greater than B's.
	%
	%

	\subsection*{Our Contributions} 
	\textbf{(1). Contextual dueling bandits under linear stochastic transitivity models: } We tackle the contextualized dueling bandits problem and extend the class of possible contextualized feedback mechanisms beyond the contextual Plackett-Luce model by allowing it to be determined by an LST model in which the utility values of the arms (choice alternatives) depend linearly on the context.
	
	\textbf{(2). A computationally efficient and near-optimal algorithm: } We construct a learning algorithm, \Algo{CoLSTIM}, which for the choice of its ``first arm'' in the duel imitates the selection mechanism of an LST model, while the choice of the second arm is tailored towards the task of regret minimization.
	%
	\Algo{CoLSTIM} enjoys under mild assumptions a bound on its average regret of order $O(\sqrt{dT}\log(T))$ matching the minimax lower bound up to logarithmic terms.
	From a practical point of view, \Algo{CoLSTIM} is computationally more efficient than current algorithms for more restrictive assumptions on the feedback mechanism and performs quite well in numerical simulations in comparison.
	%
	
	\textbf{(3). Stronger lower bound:} We show a lower bound for weak regret\,---\,a weaker form of the usual average regret\,---\,which has the same order as the stronger average regret. This result shows that it is not possible to derive stronger regret guarantees for minimizing weak regret in the contextual dueling bandits problem. 
	
	
	\paragraph{Outline.} The paper is organized as follows. 
	In Section \ref{sec_rum_basics}, we recall the linear stochastic transitivity model along with all important definitions and theoretical notions.
	We give a formal description of the contextualized dueling problem under LST in Section \ref{sec_contextual_pre_bandits} and propose with \Algo{CoLSTIM} a learning algorithm for tackling this problem, which we also analyze theoretically in the form of an upper bound on its expected average regret. 
	In an experimental study we show its empirical superiority over state-of-art algorithms for special cases of LST models in Section \ref{sec_exp}.
	Finally, we discuss the works closest related to ours in Section \ref{sec_related_work}, prior to concluding the paper with
	an outlook on future work in Section \ref{sec_conclusion}.
	All proofs are given in the supplementary materials.
	
	\paragraph{Notation.}
	For $n \in \mathbb N$, we denote by $[n]$ the set $\{1,\ldots,n\}$ and by $1_{\lbrack \cdot \rbrack}$ the indicator function.
	We write $\| \bx \|_{A} = \sqrt{\bx^\top A \bx}$ for any $\bx\in \R^d$ and any positive semi-definite matrix $A\in \R^{d \times d},$ while $\| \bx \| = \| \bx \|_{\bI_d}$  for $\bI_d$ being the $d\times d$ identity matrix. 
	For symmetric matrices $A,B$, we write $A \leq B$  if $B-A$ is positive semi-definite. 
	Let $\cB_r(p)$ be the $\ell_p$ ball of radius $r\geq 0$, i.e., $\cB_r(p) = \{\x \mid \|\x\|_p \leq r\}$, where $\|\cdot\|_p$ denotes the standard $\ell_p$-norm.
	
	\section{Linear Stochastic Transitivity} \label{sec_rum_basics}

	The linear stochastic transitivity (LST) model is a class of parameterized probability models describing the outcome of a comparison between two choice alternatives from a set of $n$ choice alternatives.
	An LST model has two parameters: An $n$-dimensional parameter $\bu=(u_1,\ldots,u_n)^{\top}\in \mathbb R^n,$ where each component represents the (latent) \emph{utility} of a choice alternative, and another functional parameter in the form of a symmetric cumulative distribution function\footnote{A cumulative distribution function is symmetric if the corresponding probability distribution is symmetric around a specific value, usually the mean.} $F:\R \to [0,1]$ called \emph{comparison function.}
	According to an LST model, the probability that alternative $i$ is preferred over alternative $j,$ denoted by $i\succ j,$  is 
	\begin{align} \label{defi_LST_prob}
		\mathbb{P}(1_{\lbrack i \succ j \rbrack } = 1) =	F( u_i - u_j  ).
	\end{align}
	\paragraph{Sorting Perturbed Utilities.}
	One equivalent way to obtain the probability in \eqref{defi_LST_prob} is as follows.
	Suppose $\epsilon_i,\epsilon_j$ are two iid random variables with distribution $G$ and such that $\epsilon_j-\epsilon_i \sim F.$
	%
	Now let  $v_i$ be the \emph{perturbed utility} of $i$ by adding the noise term $\epsilon_i$ to $u_i$.
	If $F$ is continuous, then the probability that $i$ is the choice alternative with the higher perturbed utility is exactly as in \eqref{defi_LST_prob}:
	$$  \mathbb{P}\Big( i = \argmax_{k=i,j} v_k \Big) = \mathbb{P}( \epsilon_j - \epsilon_i \leq u_i - u_j )=	F( u_i - u_j  ). $$
	This process of sorting the perturbed utilities is a fairly intuitive way to model how a decision about preferences between two options has come about by a decision maker such as a human user or the environment. 
	It will also be exactly this process based on the \emph{perturbation distribution} $G$ that will lead to the key algorithmic idea of our approach. 
	\paragraph{Examples.} Specific choices of the perturbation distribution $G$ have gained much popularity, as the corresponding comparison function $F$ has a known form:
	\begin{itemize}
		[noitemsep,topsep=0pt,leftmargin=4mm]
		\item \emph{Bradley-Terry-Luce (BTL) model} --- Setting $G$ to be the standard Gumbel distribution, it is well known that the difference of two iid standard Gumbel distributed random variables is standard logistic distributed.
		%
		In this case, we have $	F( u_i - u_j  ) = \exp(u_i)/(\exp(u_i)+\exp(u_j)).$ 
		\item \emph{Thurstone-Mosteller model} --- The distribution of the difference of two iid standard Gaussian distributed random variables is  Gaussian  with zero mean and variance $2$. 
		Hence, if $G$ is the standard Gaussian distribution, it holds that  $	F( u_i - u_j  ) = \Phi(\nicefrac{(u_i-u_j)}{\sqrt{2}}),$ where $\Phi$ is the cumulative distribution function of a standard Gaussian.
		\item \emph{Exponential Noise} --- If $G$ is the exponential distribution with rate $\lambda>0,$ then $F$ is the cumulative distribution function  of a  Laplace distribution with location $0$ and scale $\lambda^{-1}.$
		Thus, if $\mathrm{sgn}$ denotes the sign function, then \\ $	F( u_i - u_j  ) = \frac12 + \frac12 \mathrm{sgn}(u_i-u_j)(1-\exp(-\nicefrac{|u_i-u_j|}{\lambda})).$ 
	\end{itemize}

	\paragraph{Context Information.}
	In order to take context information $\vec{x}_i \in \mathbb{R}^d$ about the $i^{th}$ choice alternative into account, we follow the approach by \citet{cheng2010label,ScHu18} for the BTL model\footnote{Strictly speaking, these works consider a Plackett-Luce model, which is a generalization of the BTL model, see \cite{hunter2004mm}.} and replace the constant latent utility $u_i$ by a linear function of the features, leading to a contextualized LST (CoLST) model. 
	Formally, given $n$ choice alternatives, which define a contextual decision problem, we summarize all context vectors $\vec{x}_1, \ldots , \vec{x}_n$ in a matrix $\bX \in \mathbb{R}^{d \times n}$ and define the context-dependent utilities of the choice alternatives via
	\begin{align} \label{def:utility_param}
		u_i(\bX) =  \btheta^{\top} \vec{x}_{i} = \langle \btheta, \vec{x}_{i}  \rangle, \quad \forall i = 1,\ldots,n,
	\end{align} 
	where $\btheta \in \R^d$ is some weight vector.
	With this, the probability that alternative $i$ is preferred over $j$ given the context $\bX$, according to a CoLST model, is 
	\begin{align} \label{defi_LST_context_prob}
		\mathbb{P}(1_{\lbrack i \succ j \rbrack } = 1 \, | \,  \bX) =	F( u_i(\bX) - u_j(\bX)  ) .
	\end{align}
	The log-likelihood function $\l \left( \cdot \given 1_{\lbrack i \succ j \rbrack },\{i,j\} ,\bX \right)$ (in short just $\l(\cdot)$) for an observation $(1_{\lbrack i \succ j \rbrack },\{i,j\} ,\bX)$ is 
	\begin{align} \label{def:log_likelihood_form1}
		\begin{split}
			\l \left( \btheta \right) 
			%
			%
			&= 1_{\lbrack i \succ j \rbrack } \log\big( F (\btheta^{\top} (\bx_{i} -  \bx_{j}) ) \big) 
			+(1-1_{\lbrack i \succ j \rbrack } ) \log\big(F (\btheta^{\top} (\bx_{j} -  \bx_{i}) )\big).
		\end{split}
	\end{align}  
	The log-likelihood function can be expressed in a more convenient way, if $F$ is of an exponential family type, i.e., if for $Y=1_{\lbrack i \succ j \rbrack }$ we can write $\mathbb{P}(Y \, | \,  \bX)$ as
	\begin{align} \label{defi_exp_fam_likelihood}
		%
		\exp\left(\frac{Y \langle \btheta, \bz_{i,j} \rangle - \tilde{F}(\langle \btheta, \bz_{i,j} \rangle) }{v(\eta)} + c(Y ,\eta)			\right),
	\end{align}
	where $\bz_{i,j}= \bx_{i} -  \bx_{j},$ $\tilde F$ is the anti-derivative of $F,$ $\eta$ is some scale parameter and $v$ as well as $c$ are normalization functions.
	The comparison functions $F$ in the examples above all admit such a representation.

	\section{Contextual Dueling Bandits} \label{sec_contextual_pre_bandits}
	In this section, we first introduce the problem setting prior to the suggested learning algorithm for this setting.

	\subsection{Problem Setting}  
	We consider a set of $n \in \mathbb N_{\geq 2}$ available choice alternatives that we refer to as \emph{arms}, and simply denote them by their index: $\cA = \left\{ 1, \dots , n \right\}$.
	The learning problem proceeds in a time horizon $T,$ where in each time step $t\in \{1,\ldots,T\}$, the learner observes a  context  $\bX_t =(\bx_{t,1} \dots \, \bx_{t,n})$ with $\bx_{t,i}\in \X$ for any arm $i,$  where $\X\subset\R^d$ is the context space, e.g., $\X = \cB_1(2) = \{\bx \mid \| \bx \| \leq 1\}$ could be the $\ell_2$-unit ball of radius $1$.
	%
	Each vector $\bx_{t,i}$ encodes features of the context in which an arm must be chosen, but possibly also of the arm $i$ itself. 
	In other words, $\bx_{t,i}$ contains properties of both the context (determined by the environment) and the arm, for instance obtained by a joint feature map.
	In what follows it will turn out to be convenient to consider the contrast vectors $\bz_{t,i,j} = \bx_{t,i} - \bx_{t,j}$ and assume that we equivalently obtain in each time step the $d \times \binom{n}{2}$ dimensional contrast matrix 
	$$\bZ_t = ( \bz_{t,1,2} \,  \bz_{t,1,3} \ldots  \bz_{t,1,n}  \, \bz_{t,2,3} \ldots \bz_{t,n-1,n} ).$$ 
	After observing the context (or contrast) information $\bX_t$ (or $\bZ_t$), the learner selects a pair of  arms $S_t=(i_t,j_t) \in [n]^2,$ whereupon the learner obtains preference feedback, which is either $i_t \succ j_t$ or $j_t \prec i_t,$ i.e., $i_t$ is preferred over $j_t$ or the opposite. 
	We assume that the feedback is generated by means of a CoLST model with a \emph{known} perturbation distribution $G^*$ (and corresponding comparison function $F^*$) and an \emph{unknown} weight parameter $\btheta^*.$ 
	Formally, let $Y_t = 1_{\lbrack i_t \succ j_t \rbrack}$ be the binary feedback, then for any timestep $t \in [T]$,  
	\begin{align} \label{defi_feedback_context}
		\begin{split}
			%
			%
			\mathbb{P}(Y_t = 1 \, | \, \bZ_t) 
			&= F^*\big( \langle \bz_{t,i_t,j_t}, \btheta^* \rangle  \big) 
			= F^*( u^*_{i_t}(\bX_{t})  - u^*_{j_t}(\bX_{t})),
		\end{split}
	\end{align}
	where  $u^*_{i}(\bX_{t}) = \bx_{t,i}^\top \btheta^*.$

	The goal of the learner is to select, in each time step $t$, a pair of arms $S_t=(i_t,j_t)$ involving the arm which is best for the current context $\bX_t.$
	In the realm of preference-based multi-armed bandits, the notion of a best arm can be defined in various ways \cite{bengs2021preference}.
	In our setting, where we assume choices to be guided by a linear stochastic transitivity model (with fixed but unknown weight parameter $\btheta^*$), it is natural to define the best arm for a time step $t$ by the arm with the highest (latent) utility  (see \eqref{def:utility_param}).
	More specifically, the best arm for the current time step $t$ is 
	\begin{align}
		i^*(t) = \arg \max\nolimits_{i \in \cA } u_{i}^*(\bX_t). 
	\end{align}
	With this, one can leverage  the two most prevalent notions of instantaneous regret a learner suffers for selecting a pair of arms $S_t= (i_t,j_t)$ at time $t$ from the non-contextual dueling bandits, namely the average regret and the weak regret:
	\begin{align*}
		r_t^a(S_t,\bX_t) &= \frac{2 u_{i^*(t)}^*(\bX_t) - u_{i_t}^*(\bX_t)  - u_{j_t}^*(\bX_t)}{2},  \\
		r_t^w(S_t,\bX_t)  &=  u_{i^*(t)}^*(\bX_t) - \max\nolimits_{k \in S_t} u_{k}^*(\bX_t) \, .
		%
	\end{align*}
	Thus, the average and weak cumulative regret for selecting pairs $(S_t)_{t\in[T]}$ for context information $(\bX_t)_{t\in[T]}$ during the time horizon $T$ are, respectively,
	\begin{align}\label{regret_def} 
		\begin{split}
			R_T^a((S_t)_{t\in[T]}) \, &= \, \sum\nolimits_{t=1}^{T} r_t^a(S_t,\bX_t), \\
			R_T^w((S_t)_{t\in[T]}) \, &= \, \sum\nolimits_{t=1}^{T} r_t^w(S_t,\bX_t).
			%
			%
			%
		\end{split}
		%
	\end{align}
	%
	%
	%
	%
	%
	%
	It holds that $R_T^w((S_t)_{t\in[T]}) \leq R_T^a((S_t)_{t\in[T]}),$ so that a learner with theoretical guarantees for the average regret satisfies theoretical guarantees for the weak regret as well.
	However, being a weaker notion of regret, the weak regret may permit stronger theoretical guarantees than average regret. 
	\citet{saha2020regret} has shown a lower bound of order $\Omega(\sqrt{dT})$ for the average regret, and the following result shows that this bound also holds for the weak regret, indicating that stronger theoretical guarantees cannot be obtained for weak regret learners.
	\begin{theorem} \label{theorem:lower_bound_weak_regret}
		For any learning algorithm $\cA$ for the contextual dueling bandits problem under linear stochastic transitivity models in dimension $d$, there exists an instance of the problem characterized by a weight vector $\btheta^* \in \cB_1(2)$, context space $\X \subseteq \cB_1(\infty),$  and a comparison function $F:\R \to [0,1]$, such that the expected weak regret incurred by $\cA$ in any $T > \max(d,16)$ rounds is 
		$$\mathbb{E} \big[ R_T^w((S_t^\cA)_{t\in[T]}) \big]=\Omega\big(d\sqrt{T}\big).$$
		%
		%
		Further, if we restrict $\X \subseteq \cB_1(2)$, then 
		$$\mathbb{E} \big[ R_T^w((S_t^\cA)_{t\in[T]}) \big] = \Omega\big(\sqrt{dT}\big).$$ 
	\end{theorem}
	It is worth noting that our proof of Theorem \ref{theorem:lower_bound_weak_regret} does not use a reduction to contextual linear bandits as the proof by \citet{saha2020regret} for the average regret (see Theorem 10), but is rather based on a direct proof technique (see Section \ref{sec:weak_regret_lower_bound_proof} in the supplementary material).

	
	\subsection{CoLST Imitator} \label{sec_rums_algo}

	At the core of the learning task is the estimation of the unknown weight parameter $\btheta^*$, which basically determines the underlying CoLST model of the feedback mechanism.
	Assuming that we have a compact parameter space $\Theta$ such that $\btheta^* \in \Theta,$ the arguably most natural way to derive an estimate is to use the maximum likelihood estimate (MLE):
	%
	\begin{align} \label{def_MLE}
		\hat \btheta_t \in \argmax\nolimits_{\btheta \in \Theta}  \sum\nolimits_{s=1}^{t-1} \, \l \left( \btheta \given Y_s ,\{i_s,j_s\} ,\bX_s \right).
	\end{align}
	If the log-likelihood function is of the form \eqref{defi_exp_fam_likelihood}, then $\hat \btheta_t$ can be computed by solving
	\begin{align} \label{def_MLE_variant}
		0 = \sum\nolimits_{s=1}^{t-1} \big( Y_s - F(  \langle \bz_{s,i_s,j_s} , \btheta \rangle  ) \big) \bz_{s,i_s,j_s}.
	\end{align} 
	Equipped with an estimate of the underlying weight parameter, the question now becomes how to choose an appropriate pair of arms to deal with the exploration-exploitation trade-off. 
	Our approach is essentially to imitate the choice process of the underlying CoLST model using the current estimator and their confidence widths to generate perturbed, context-dependent utilities giving rise to our CoLST Imitator (\Algo{CoLSTIM}) algorithm (see Algorithm \ref{alg:CoLSTIM}).

	\begin{algorithm}
		\caption{\Algo{CoLSTIM}} \label{alg:CoLSTIM}
		\begin{algorithmic}[1]
			\STATE \textbf{Input:} Exploration length $\tau>0,$   perturbation distribution $G$ with induced comparison (or cumulative distribution) function $F,$ threshold $C_{\mathrm{thresh}}>0, $
			coupling probabilities $(p_t)_t,$
			confidence width constant $c_1>0$
			\STATE \textbf{Initialization:} Randomly choose $\{i_t,j_t\} \subset [n]$ for $\tau$ many time steps
			\STATE Set $\bM_{\tau+1}=\sum_{s=1}^\tau \bz_{s,i_s,j_s} \bz_{s,i_s,j_s}^\top$
			\FOR{$t=\tau+1,2,\ldots,T$} 
			\STATE Observe context vectors $\bX_t =(\bx_{t,1} \ldots \bx_{t,n})$
			\STATE Compute MLE $\hat{\btheta}_t$ via \eqref{def_MLE} (or \eqref{def_MLE_variant})
			\STATE Sample $B_t \sim \mathrm{Ber}(p_t)$
			\IF{ $B_t=1$}
			\STATE Sample $\tilde \epsilon_{t,i} \sim G$ for each $i\in[n]$
			\ELSE
			\STATE Sample $\tilde \epsilon \sim G$ and set $\tilde \epsilon_{t,i}  = \tilde \epsilon$ for each $i\in[n]$
			\ENDIF
			\STATE 	$\epsilon_{t,i} = \min( C_{\mathrm{thresh}}, \max(  - C_{\mathrm{thresh}} , \tilde \epsilon_{t,i}  )  ) \ \forall i \in [n]$	
			\STATE $ i_t = \arg\max_{i\in[n]} \bx_{t,i}^\top \hat \btheta_t + \epsilon_{t,i}\|\bx_{t,i} \|_{\bM_t^{-1}}  $
			\STATE $ j_t =  \arg\max_{i\in[n] } \langle \bz_{t,i,i_t} , \hat \btheta_t \rangle  + c_1 \| \bz_{t,i_t,i} \|_{\bM_{t}^{-1}}  $
				%
				\STATE Choose $(i_t,j_t)$ and observe $Y_t = 1_{\lbrack i_t \succ j_t \rbrack}  $ 
				\STATE Update  $\bM_{t+1} \leftarrow \bM_{t} +  \bz_{t,i_t,j_t} \bz_{t,i_t,j_t}^\top$
				\ENDFOR
			\end{algorithmic}
		\end{algorithm}

		More specifically, we first generate additive noise terms by sampling for each arm a random observation of the underlying perturbation distribution $G$ and multiply it with the corresponding confidence width.
		Due to technical reasons we need to (a) threshold the generated perturbation variables from above (and below) by means of some constant $C_{\mathrm{thresh}}$ (and $-C_{\mathrm{thresh}}$), and (b) use the same perturbation variable with a sufficient high probability (\emph{coupling}).
		These additive noise terms are added to the current utility estimates resulting in perturbed, context-dependent utilities, which in turn specify our imitated CoLST model.
		The arm having the largest perturbed context-dependent utilities is used as the ``first arm'' $i_t$ for the action pair, as this one is most likely to win any duel according to our imitated CoLST model.
		Roughly speaking, the multiplication of the generated perturbation variables with the confidence widths (of the current estimate) takes the degree of accuracy of the imitated CoLST model into account.
		The ``second arm'' $j_t$ is chosen as the first arm's toughest competitor, i.e., the arm that has the highest (optimistic) chance to win the duel against the latter according to the current upper confidence width (represented by $ c_1 \| \bz_{t,i_t,i} \|_{\bM_{t}^{-1}}$). 
		Such a choice mechanism for the second arm is common in the realm of non-contextual dueling bandits and is tailored towards the task of average regret minimization \citep{bengs2021preference}.
		%

		The underlying idea of choosing the first arm is inspired by randomized exploration strategies used in the follow-the-perturbed-leader approach  \citep{kim2019optimality} or the RandUCB algorithm \cite{vaswani2020old} for numerical bandit problems.
		However, in our preference-based setting, we have to deal with a more complex action space leading to a different theoretical analysis. 
		In addition, unlike the numerical setting, we do not necessarily need to adopt the optimism in the face of uncertainty principle in the sense that only positive perturbation values need to be sampled.
		Finally, thanks to the underlying CoLST model, we have a natural candidate for the perturbation distribution.

		\subsubsection{Near-optimality}
		For our theoretical analysis of \Algo{CoLSTIM}, we make the following assumptions: 
		\begin{itemize}
			 \item [\textbf{(A1)}]	 $\X = \cB_1(2)$ and there exist basis vectors $\{b_j\}_{j=1}^d \subset \{\bx_{t,i}\}_{i=1}^n$ such that $\rho \bI_d \leq \sum_{j=1}^d b_j b_j^\top$ for some $\rho>0$, i.e., the context vectors span the $d$-dimensional Euclidean space. 
			\item [\textbf{(A2)}] There exists some $L>0$ such that the derivative of the comparison function $F^*$ is $L$-Lipschitz continuous, i.e., $| (F^*)'(x) - (F^*)'y) |\leq L |x-y|$ for any $x,y\in \R.$ 
			Moreover,
			$$ \mu := \inf\big\{ (F^*)'(	\bx^\top \btheta) \, \big| \, \| \bx \|\leq 2, \, \| \btheta - \btheta^* \|\leq 1			\big\} >0.$$
		\end{itemize}
		%
		
		%
		%
		
		These assumptions are quite common in the realm of contextual bandits \citep{li2017provably,vaswani2020old}.
		Note that assumption \textbf{(A2)} holds for the comparison functions of the prominent LST models such as the Bradley-Terry-Luce and the Thurstone-Mosteller model (see Section \ref{sec_rum_basics}).
		We obtain the following theoretical guarantees for $\Algo{CoLSTIM}$ (proven in Section~\ref{sec_proof_colstim} in the supplementary material).
		\begin{theorem} \label{theorem_regret_bound}
			Let $c_1= \frac{1}{2\mu}\sqrt{d \log(T/d) +2\log(T)}$ and $\tau=d + \max\{ \nicefrac{d^2 \log(T)}{\mu^2 \rho}  ,d/\rho \}.$
			Further, let $ c_2 \in (0,c_1] $ be some constant, and for any time step $t$, set $p_t = \min\left(1, \, \frac{\sqrt{2d}}{2\sqrt{t - \tau}} \Big(	3 \, c_1 + c_2	\Big) \sqrt{\log\left( \frac{2T}{d} \right) } \, \right).$
			For any $ C_{\mathrm{thresh}} \in(0,c_2)$, it holds that the expected cumulative average regret of $\Algo{CoLSTIM}$ is in 
			$O(d\sqrt{T}\log(T))$  for any CoLST model with weight vector $\btheta^*$ such that $\|\btheta^*\|\in \cB_1(2)$ and comparison function $F^*$ satisfying assumption \textbf{(A2)}, and any $(\bX_t)_{t\in[T]}$ satisfying assumption \textbf{(A1)}.
			%
			%
			%
		\end{theorem}
		
		%
		%
		If we make an additional assumption on the smallest eigenvalues of the Gram matrix $\bM_t,$ we can show a bound  $\tilde O( \sqrt{dT})$ almost matching (up to logarithmic terms) the lower bound $\Omega(\sqrt{dT})$ shown by \citet{saha2020regret}.
		The proof is deferred to Section \ref{sec_proof_colstim} in the supplementary material.
		\begin{corollary} \label{corollary_colstim}
			Under the assumptions of Theorem \ref{theorem_regret_bound} and if $\sum_{t=\tau+1}^T \lambda_{\min}^{-1/2}(\bM_t)\leq c \sqrt{T},$ where $c$ is some positive constant and $ \lambda_{\min}(A) $ denotes the smallest eigenvalue of a square matrix $A,$ the expected cumulative average regret of $\Algo{CoLSTIM}$ is in $O( \sqrt{dT}\log(T)).$
			%
			%
			%
		\end{corollary}
		The condition on the Gram matrix in the Corollary \ref{corollary_colstim} is satisfied if the context vectors are dense \citep{li2017provably}.

		
		\begin{algorithm}
			\caption{\Algo{Sup-CoLSTIM}} \label{alg:SUPCoLSTIM}
			\begin{algorithmic}[1]
				\STATE \textbf{Input:}  $\tau>0,$ $G,$ $F,$  $C_{\mathrm{thresh}}>0, $ $p_t\in[0,1],$ $c_1>0$
				\STATE \textbf{Initialization:} Same as in \Algo{CoLSTIM}
				%
				%
				\STATE Set $S = \lfloor \log_2 T \rfloor,$   $\Psi^{(0)} = \emptyset, \Psi^{(1)} = \ldots = \Psi^{(S)} = [\tau]$
				\FOR{$t=\tau+1,2,\ldots,T$} 
				\STATE Observe context vectors $\bX_t =(\bx_{t,1} \ldots \bx_{t,n})$
				\STATE Set $s = 1$ and $A_t^{(s)} = [n]$
				\WHILE{$i_t,j_t$ not found}
				\STATE Compute MLE $\hat{\btheta}_t^{(s)}$ via \eqref{def_MLE} (or \eqref{def_MLE_variant}) using only data from the time steps in $\Psi^{(s)}$
				\STATE Set $\bM_{t}^{(s)}=\sum_{ l \in \Psi^{(s)}} \bz_{l,i_l,j_l} \bz_{l,i_l,j_l}^\top$
				\STATE Set $ w_{i,j}^{(s)}(\bX_t) = c_1 \| \bz_{t,i,j} \|_{(\bM_{t}^{(s)})^{-1}}   \forall i,j \in [n]$
				\IF{ $w_{i,j}^{(s)}(\bX_t) \leq 1/\sqrt{T}, \quad \forall i,j \in A_t^{(s)}$}
				\STATE Sample $\epsilon_{t,i} \ \forall i \in [n]$	as in lines 7--13 of Alg.\ref{alg:CoLSTIM} 
				\STATE $\Psi^{(0)} = \Psi^{(0)} \cup \{t\}$
				\STATE $ i_t = \arg\max\limits_{i\in A_t^{(s)} } \bx_{t,i}^\top \hat \btheta_t^{(s)} + \epsilon_{t,i}\|\bx_{t,i} \|_{(\bM_{t}^{(s)})^{-1}}  $
				\STATE $ j_t =  \arg\max\limits_{i\in A_t^{(s)} } \langle \bz_{t,i,i_t} , \hat \btheta_t^{(s)} \rangle  + w_{t,i_t,j}^{(s)}(\bX_t) $
				
				\ELSIF{$w_{i,j}^{(s)}(\bX_t) \leq 1/2^s, \quad  \forall i,j \in A_t^{(s)}$}
				
				\STATE $A_t^{(s+1)} = \{  i \in  A_t^{(s)} \, | \, \bx_{t,i}^\top  \hat \btheta_t^{(s)} + 2^{-s} \geq  \max\nolimits_{j \in A_t^{(s)}} \bx_{t,j}^\top  \hat \btheta_t^{(s)} \}$
				\STATE $s \leftarrow s+1$
				\ELSE 
				\STATE $\Psi^{(s)} = \Psi^{(s)} \cup \{t\}$
				\STATE Choose $(i_t,j_t)$ uniformly at random from $\{i,j \in A_t^{(s)} \, | \, w_{i,j}^{(s)}(\bX_t) > 1/2^s \}$
				\ENDIF
					\ENDWHILE
					\STATE Choose $(i_t,j_t)$ and observe $Y_t = 1_{\lbrack i_t \succ j_t \rbrack}  $ 
					%
					%
					\ENDFOR
				\end{algorithmic}
			\end{algorithm}

			Note that the theoretical results in fact do not depend on the knowledge of the true perturbation distribution $G^*,$ but rather it is sufficient to know the true comparison function $F^*.$
			Finally, it is worth noting that the threshold $C_{\mathrm{thresh}}$ can be made arbitrarily small, while the theoretical results will still hold (as long as $c_2$ is not even smaller). 
			In this case, the choice of the first arm $i_t$ would essentially correspond to a greedy choice, which prevents exploration from the outset.
            However, thanks to the (optimistic) choice mechanism of the second arm $j_t$, the algorithm would still explore (i.e., learn about the structure of the bandit problem) and not “get stuck”. 
			
			\subsubsection{Sup-CoLST Imitator}
			By adapting the technique introduced by \citet{Au02} for the underlying learning scenario (cf.\ \citep{ChLiReSc11,li2017provably,xuenearly}), we can extend the $\Algo{CoLSTIM}$ algorithm to \Algo{Sup-CoLSTIM} (Algorithm \ref{alg:SUPCoLSTIM}) in order to obtain a regret bound of order $\tilde O(\sqrt{d \, T \log(n)})$ without making an additional assumptions on the Gram matrix as in Corollary \ref{corollary_colstim}.
			The idea is to embed the choice mechanism of \Algo{CoLSTIM} into a stage-wise approach which keeps track of ``sufficiently accurately estimated promising arms'' (cf.\ Section 3.2 in \cite{saha2020regret}).
			%
			
%
			\begin{theorem} \label{theorem_regret_bound_Sup_colstim}
				Let $c_1= \frac{3}{2\mu}\sqrt{2 \log( 3 n T^2)}$ and $\tau$ as in Theorem \ref{theorem_regret_bound}.
				Further, let $  c_2 \in (0,c_1] $ be some constant and $p_t$ be as in Theorem \ref{theorem_regret_bound}.
				Under the assumptions of Theorem \ref{theorem_regret_bound}, it holds for any $ C_{\mathrm{thresh}} < c_2$  that the expected cumulative average regret of $\Algo{Sup-CoLSTIM}$ is in
				$O(\sqrt{dT\log(n)}\log^{3/2}(T)).$
				%
			\end{theorem}
			To the best of our knowledge, this is the first time that the technique introduced by \citet{Au02} is combined with a randomized learning strategy and analyzed theoretically (see Section \ref{sec_proofs_sup_colstim} in the supplementary material for the proof).
			%

			%
			%

			\subsubsection{Computational Aspects}
			It is evident that the computation of the MLE involves a computationally expensive operation, as the entire history is used for this estimation step\,---\,an issue shared by most of the algorithms for logistic bandits \citep{filippi2010parametric, li2017provably, vaswani2020old} or stochastic contextual dueling bandits \cite{saha2020regret}.
			However, instead of optimizing the (log-)likelihood function based on the entire history in each time step, we could also  optimize the (log-)likelihood function more efficiently in an online manner using stochastic gradient descent.
			More precisely, we could replace line 6 in Algorithm \ref{alg:CoLSTIM} (and line 9 in Algorithm \ref{alg:SUPCoLSTIM} accordingly) by
			\begin{center}
				$\hat \btheta_t 
				\leftarrow \hat \btheta_{t-1}  + \eta_t \nabla \l \big( \hat \btheta_{t-1} \given  Y_{t-1},\{i_{t-1},j_{t-1}\} ,\bX_{t-1} \big)$
				%
			\end{center}
			for some suitable parameter $\eta_t>0$ (learning rate).
			Although the theoretical guarantees shown do not hold for this SGD variant, we do not see much of a difference regarding the regret in our experiments by using the SGD variant instead of maximizing the likelihood on the entire history of data in each time step (see Section \ref{sec:sgd_vs_mle}).
			However, we see a clear advantage of the former over the latter with respect to the cumulative elapsed time to make a choice.
			It is worth mentioning that the computational costs of \Algo{CoLSTIM}'s arm selection (lines 7--15) is quite low due to the simple form of the two maximization problems to be solved.

		\section{Experiments} \label{sec_exp}
		In this section, we present experimental results for our learning algorithm for the contextual dueling bandits setting for the two most prominent LST models, namely the Bradley-Terry-Luce (BTL) model and the Thurstone-Mosteller (TM) model.
		We compare our approach with Double-Thompson Sampling (DTS) \citep{WuLi16} and Self-Sparring
		(SS) with independent beta priors for each arm \citep{SuZhBuYu17}, which are state-of-the art algorithms for the non-contextual dueling bandits setting, as well as Maximum-Informative-Pair (MaxInP) \citep{saha2020regret}, which is suitable for the contextual dueling bandits setting under the contextualized BTL model.
		Moreover, we include the random choice strategy (Random), which is choosing the pair of arms in each time step uniformly at random.
		For DTS and SS, we used the same hyperparameters as in the corresponding experiments, while for MaxInP, we simply use $t_0= d n$ and $\eta= \sqrt{d \log(T)	},$ as the hyperparameters are not reported in the experiments by \citet{saha2020regret}. 
		Note that our choices are simplified versions of the parameters for which the theoretical guarantees hold.
		For \Algo{CoLSTIM}, we use $\tau = t_0$ and $c_1 = C_{\mathrm{thresh}}=\eta,$ as $\tau$ has a similar role as $t_0$ and $c_1$ or $C_{\mathrm{thresh}}$  have a similar role as $\eta,$ respectively.
		For the coupling probability $p_t$ of  \Algo{CoLSTIM}, we use a simplified version of the one derived in Theorem \ref{theorem_regret_bound}, namely $p_t = \min\big(1, \, \frac{d}{\sqrt{t - \tau}} \log\left( d T \right)   \big).$
		In every experiment, the performances of the algorithms are measured in terms of cumulative average regret (cf.\ \eqref{regret_def}), averaged across 100 runs and reported with their standard deviation.
		Moreover, for both MaxInP and \Algo{CoLSTIM}, we use the SGD-based variant described in Section \ref{sec_rums_algo} with a fixed learning rate of $\eta_t = 1/2.$ 
		For a comparison of the SGD-based variant and the ``full MLE'' variants see Section \ref{sec:sgd_vs_mle}.
		We omit \Algo{Sup-CoLSTIM} and Sta'D \citep{saha2020regret} since it is known that algorithms adopting the stage-wise approach of \citet{Au02} tend to perform poorly in numerical simulations.
%
		All these experiments were conducted on a machine featuring an Intel(R) Xeon E5-2670 @2.6GHz with 16 cores and 64 GB of RAM.
		\subsection{Contextual Setting}\label{synthDataExp}
		Recall that a problem instance $P$ in our setting is specified by the number of available arms $n,$ the dimension of the context vectors $d,$ the perturbation distribution $G^*$ and the weight parameter $\btheta^*,$ such that we accordingly write $P = P(n,d,G,\btheta^*).$
		Following \citet{saha2020regret}, for fixed $n,d,G^*$, we distinguish three problem scenarios with respect to the $\ell_2$-norm of the weight parameter $\btheta^*$:
		\begin{itemize}	
			%
			\item $E(d,n,G^*) = \bigcup_{\|\btheta^*\| \leq 1/\sqrt{d}} P(n,d,G^*,\btheta^*);$
			\item  $M(d,n,G^*) = \bigcup_{1/\sqrt{d} \leq \|\btheta^*\| \leq 1} P(n,d,G^*,\btheta^*);$
			\item  $H(d,n,G^*)= \bigcup_{1 \leq \|\btheta^*\| \leq \sqrt{d}} P(n,d,G^*,\btheta^*).$ 
		\end{itemize}
		which we refer to as the easy, medium and hard problem scenario, respectively.

		In each run on one of the scenarios, the weight parameters are sampled uniformly at random from the corresponding subset of the $\ell_2$-Ball, while the context vectors are sampled uniformly at random from $\cB_1(2),$ i.e., the unit $\ell_2$-Ball, for each time step $t$ within one run.

		\begin{figure}[ht!]		\vskip 0.2in
			\begin{subfigure}
				\centering
				\centering
				\includegraphics[width=.33\linewidth]{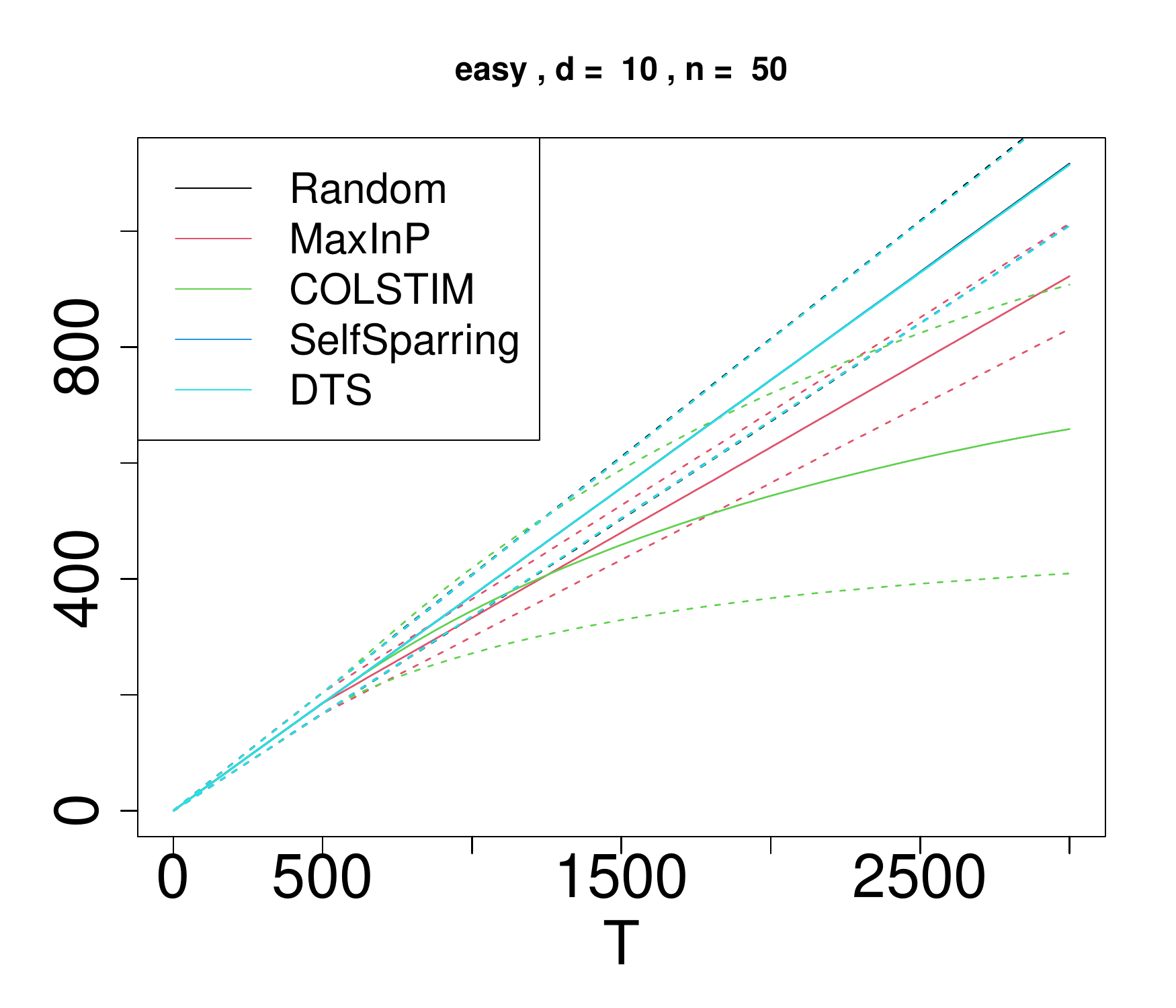}
			\end{subfigure}%
			\begin{subfigure}
				\centering
				\includegraphics[width=.33\linewidth]{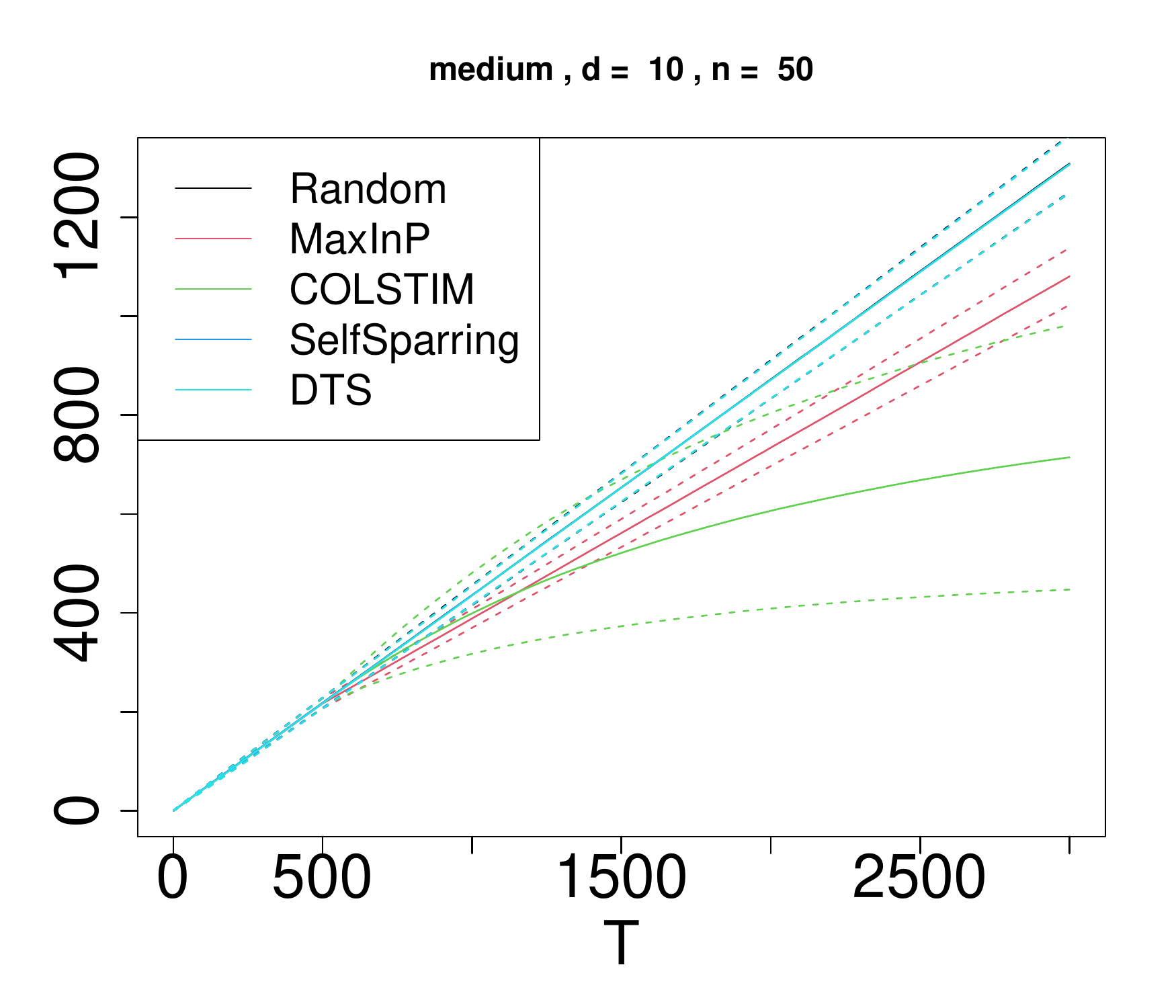}
			\end{subfigure}%
			\begin{subfigure}
				\centering
				\includegraphics[width=.33\linewidth]{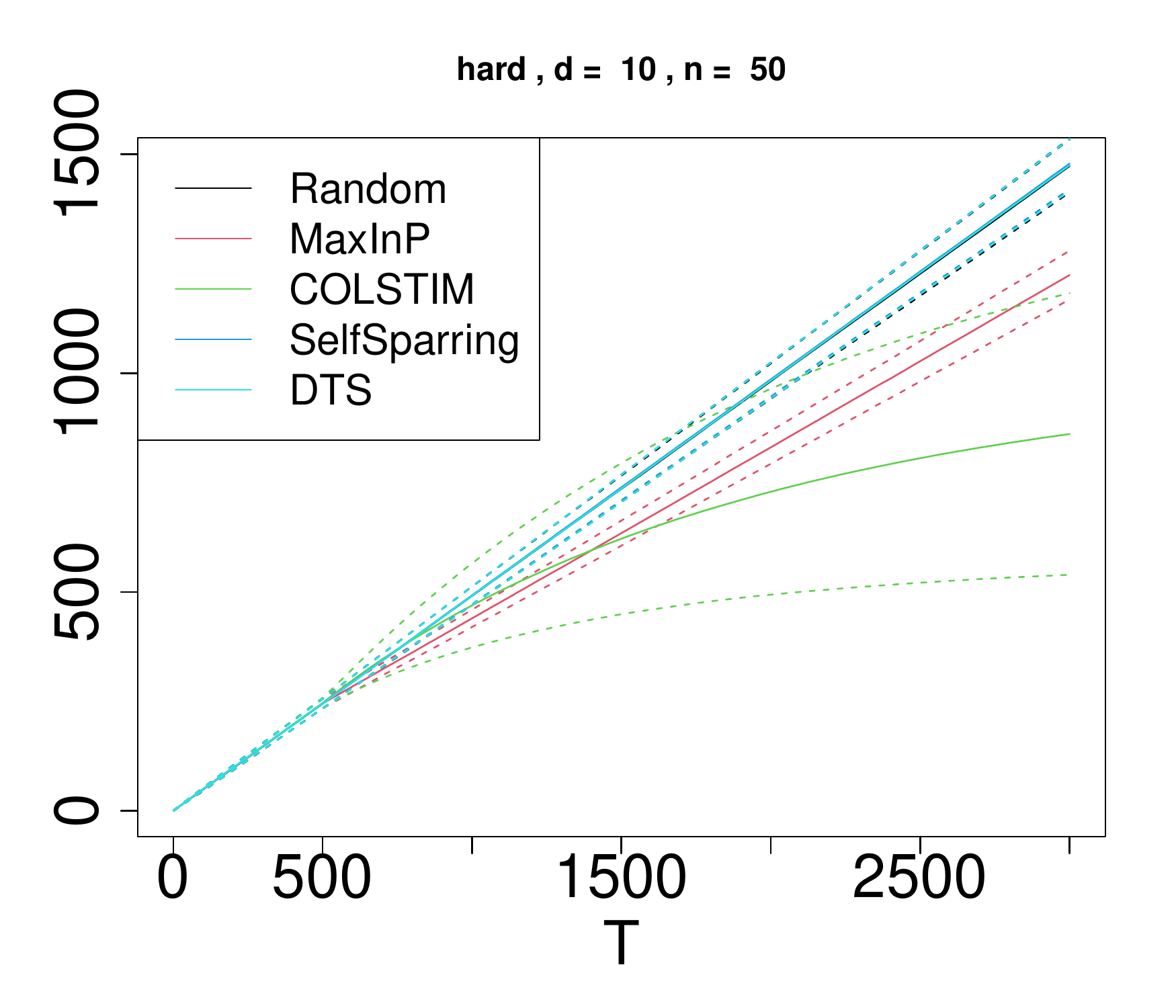}
			\end{subfigure}
			\begin{subfigure}
				\centering
				\includegraphics[width=.33\linewidth]{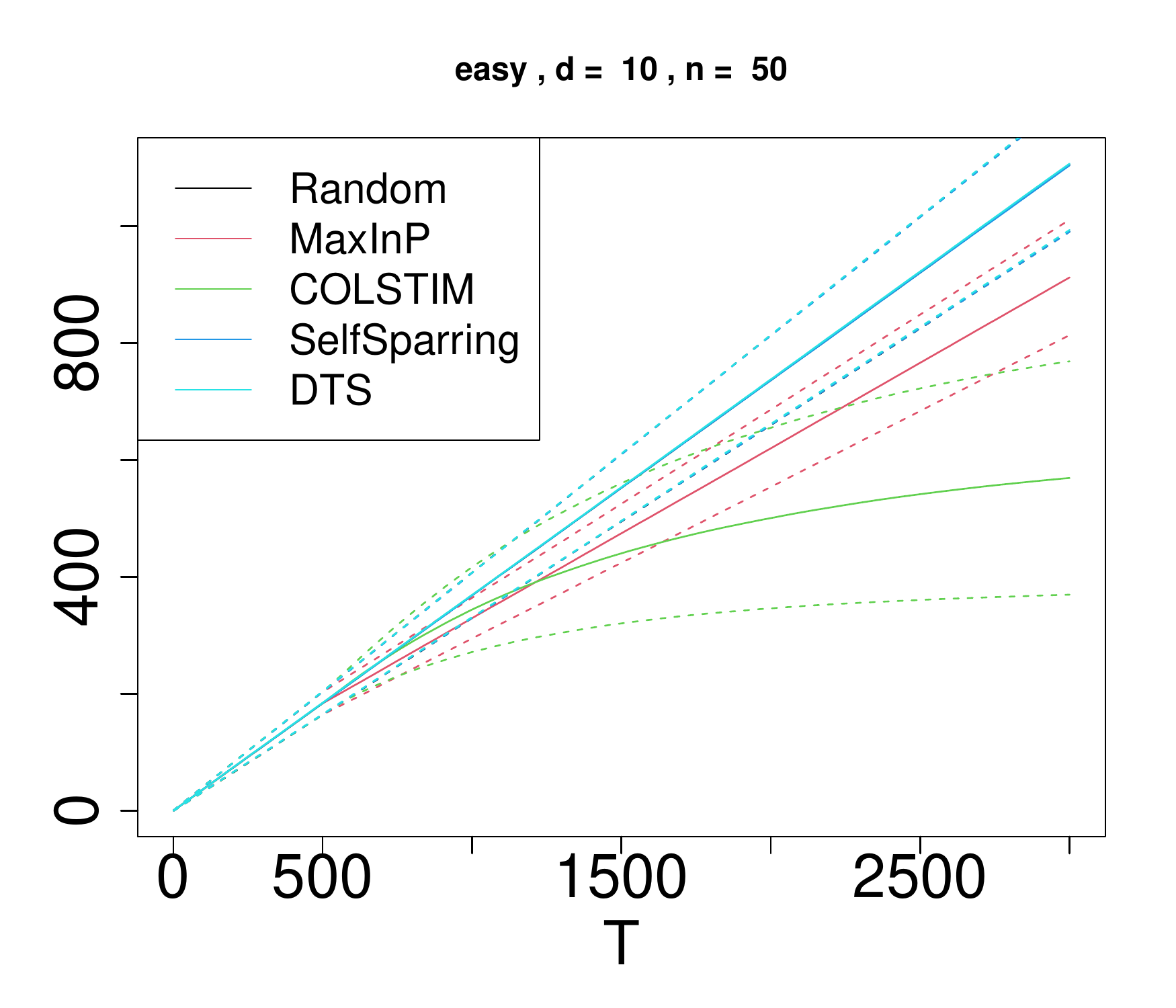}
			\end{subfigure}
			\begin{subfigure}
				\centering
				\includegraphics[width=.33\linewidth]{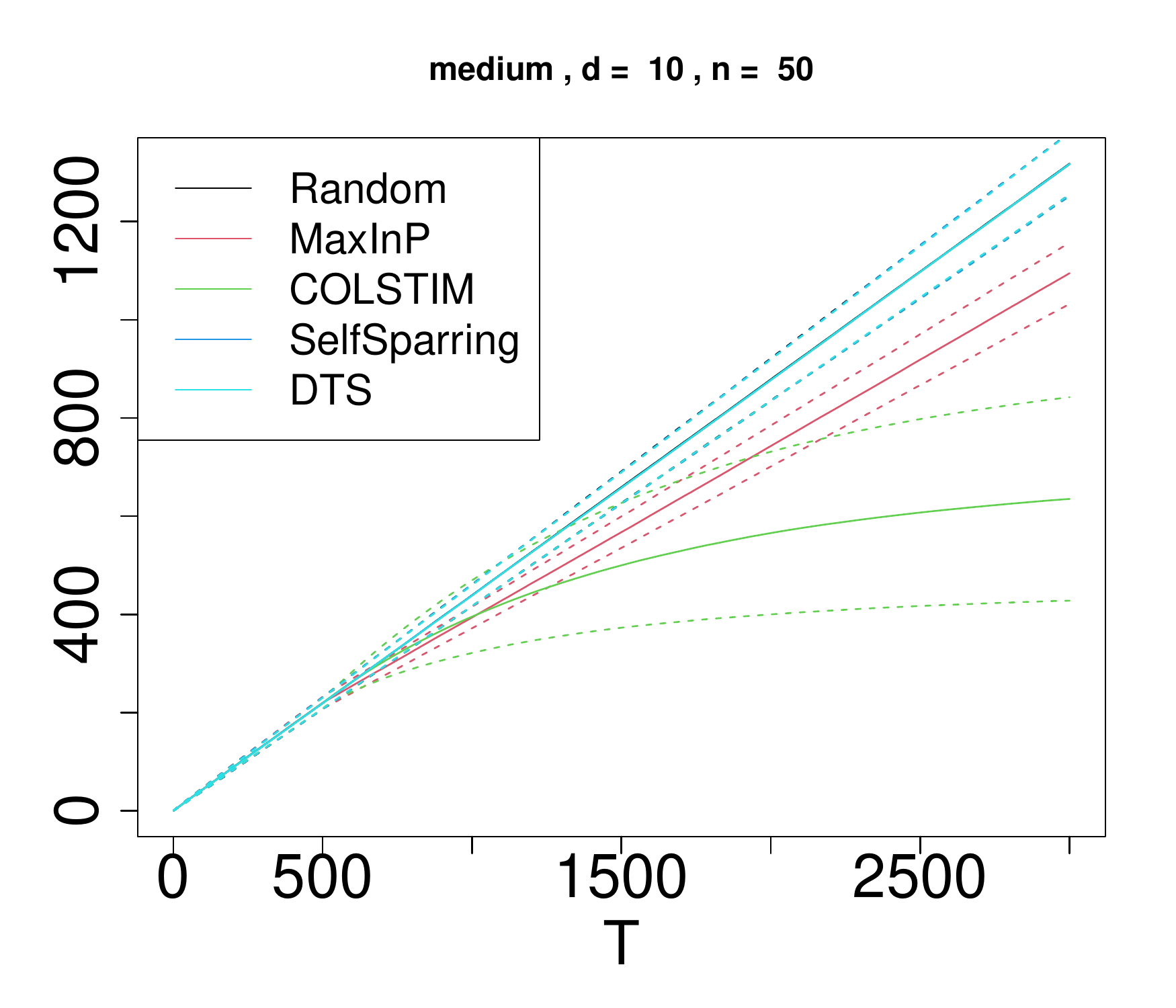}
			\end{subfigure}%
			\begin{subfigure}
				\centering
				\includegraphics[width=.33\linewidth]{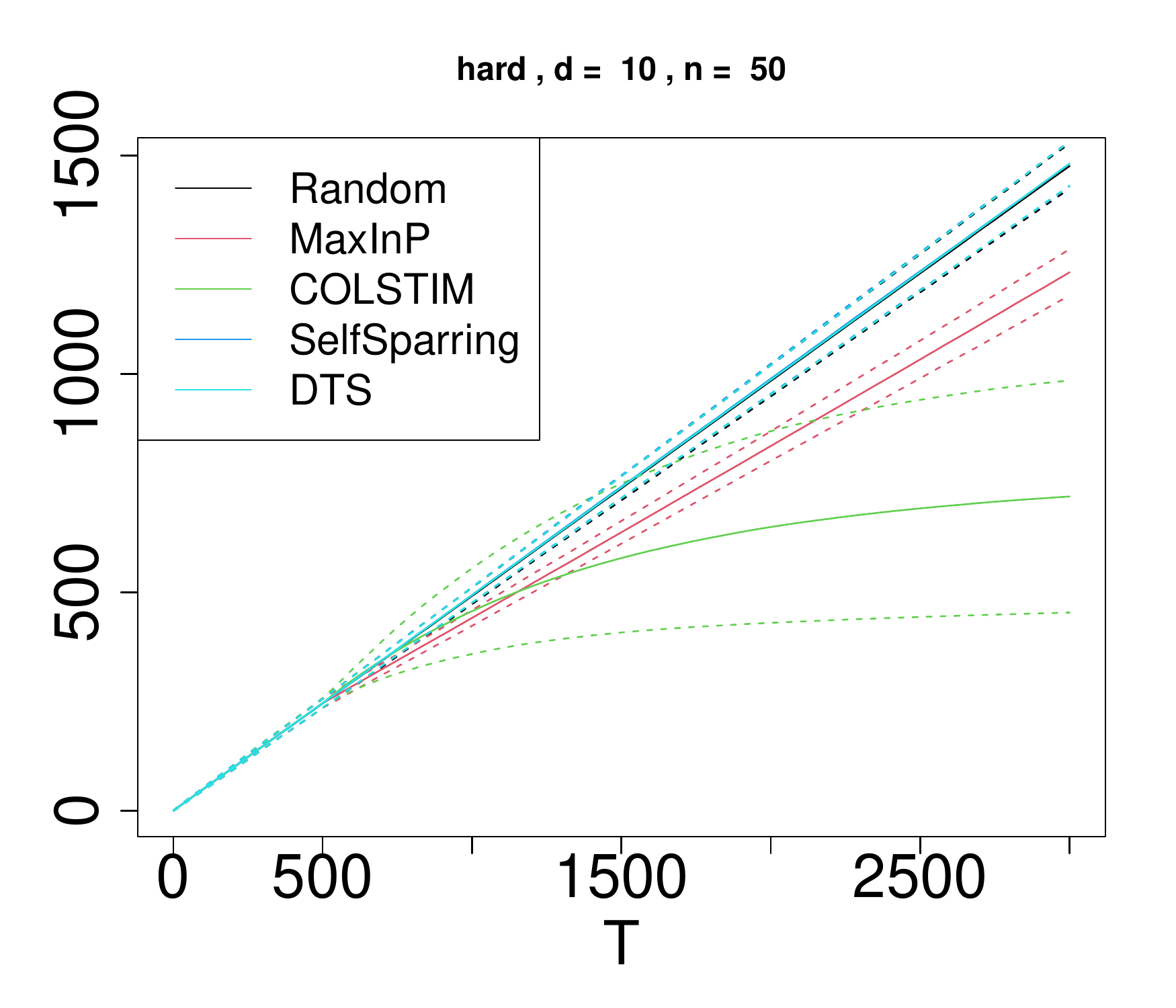}
			\end{subfigure}%
		\vskip -0.2in
			\caption{Top panel: Averaged cumulative regret of the different methods on $E(d,n,G^*)$ (left), $M(d,n,G^*)$ (middle) and $H(d,n,G^*)$ (right) for $G^*$ being standard Gumbel. 
				Bottom panel: Averaged cumulative regret of the different methods on $E(d,n,G^*)$ (left), $M(d,n,G^*)$ (middle) and $H(d,n,G^*)$ (right) for $G^*$ being the standard normal distribution.}
		\label{fig:contextual_exp}
		\vskip 0.2in
	\end{figure}

		The results for $G^*$ being the standard Gumbel distribution are illustrated in the top panel of Figure \ref{fig:contextual_exp}, while the bottom panel of Figure \ref{fig:contextual_exp} shows the results for $G^*$ being the standard normal distribution, both for $n=50$ and $d=10$.
		Our \Algo{CoLSTIM} method using the corresponding perturbation distribution $G=G^*$  outperforms the other methods in all scenarios, while MaxInP performs only slightly better than the na\"ive random selection strategy.
		Unsurprisingly, the non-contextual methods DTS and SS are hardly distinguishable from the latter.
		%
		%
		%
				%

			\begin{table}[ht]
			\caption{Averaged cumulative runtimes (in seconds) and the corresponding standard deviations (in brackets) of the different methods for the different problem scenarios.}
			\label{tab:runtimes}
			\centering
				\begin{tabular}{r||r|r|r||r|r|r}
					& \multicolumn{3}{|c||}{	$G^*$ = Gumbel} 
					& \multicolumn{3}{|c}{	$G^*$ = Gaussian}  \\
					\hline
					& $E(d,n,G^*)$ & $M(d,n,G^*)$ & $H(d,n,G^*)$ & $E(d,n,G^*)$ & $M(d,n,G^*)$ & $H(d,n,G^*)$ \\ 	
					\hline
					Random & 0.19 (0.02)  & 0.18 (0.01)   &   0.18 (0.01) & 0.19 (0.02)   & 0.18 (0.02)  & 0.18 (0.02) \\ 
					MaxInP & 164.58 (6.58)  & 155.54 (4.92)  & 155.51  (3.49)  &  159.33 (6.00) & 157.70  (8.54)   & 156.56  (7.48) \\  
					\Algo{ColSTIM} & 7.44  (0.51)  & 7.08 (0.41)  & 6.98  (0.17) & 7.09 (0.38)  &  7.05 (0.51) & 7.03  (0.48) \\ 
					SS &  9.92 (0.66) &  9.43 (0.56) & 9.39 (0.22) & 9.49 (0.50)  &  9.47  (0.67) & 9.47 (0.60) \\  
					DTS & 52.01 (2.63) &  49.73(2.04) & 50.15 (1.08) & 50.94 (2.45)  &  50.19  (2.99) & 49.93 (2.69) \\ 

				\end{tabular}
		\end{table}

		Table \ref{tab:runtimes} reports the averaged cumulative runtimes and the corresponding standard deviations of each of the considered methods for making the choice of the pairs (i.e., the MLE step for \Algo{CoLSTIM} and MaxInP are neglected).
		The results confirm the discussion in Section \ref{alg:CoLSTIM} regarding the computational efficiency of \Algo{CoLSTIM}'s choice mechanism, which, quite interestingly, is competitive to the non-contextual methods.

						\begin{figure}[ht!]
			\vskip 0.2in
			\begin{center}
				\begin{subfigure}
					\centering
					\centering
					\includegraphics[width=.32\linewidth]{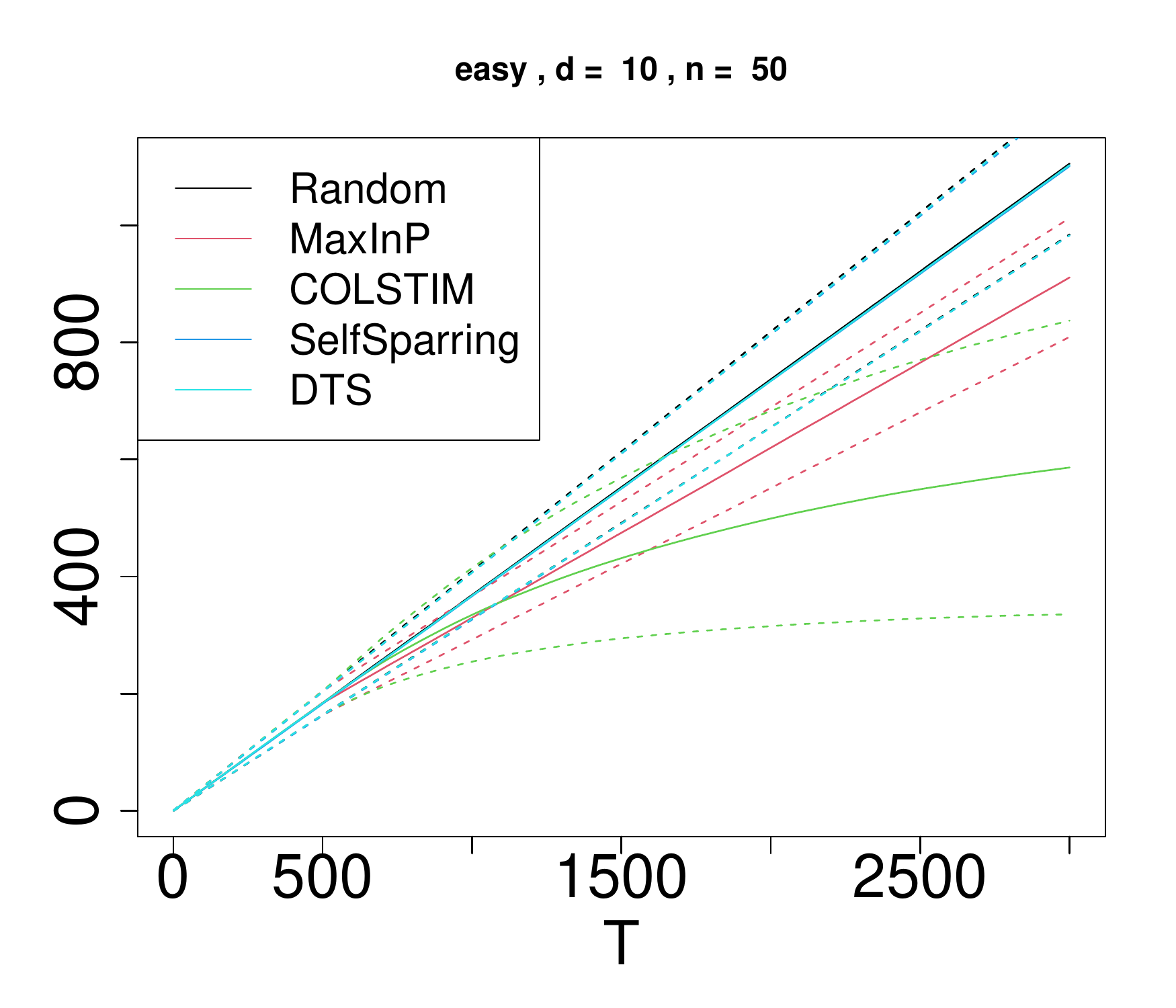}
				\end{subfigure}%
				\begin{subfigure}
					\centering
					\includegraphics[width=.32\linewidth]{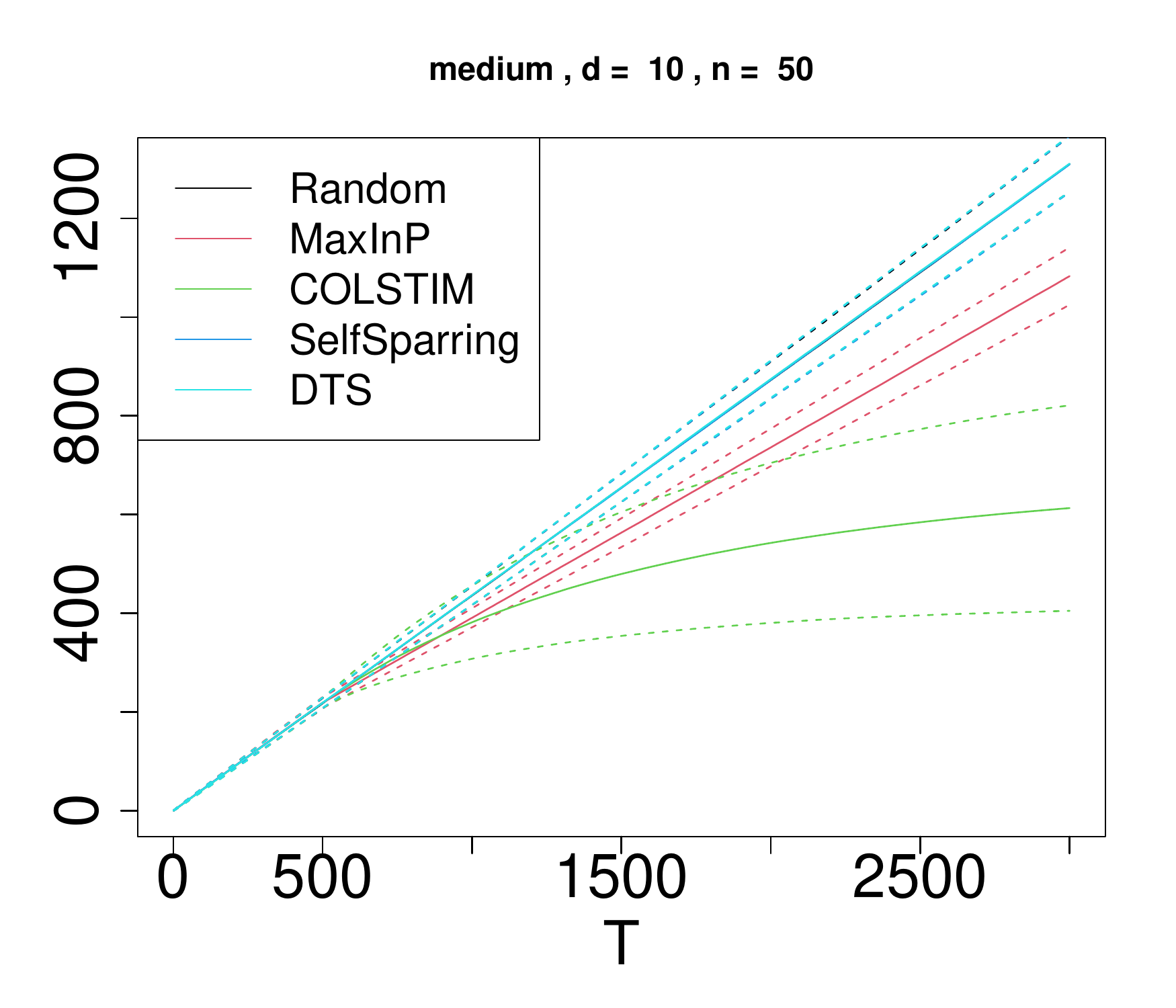}
				\end{subfigure}%
				\begin{subfigure}
					\centering
					\includegraphics[width=.32\linewidth]{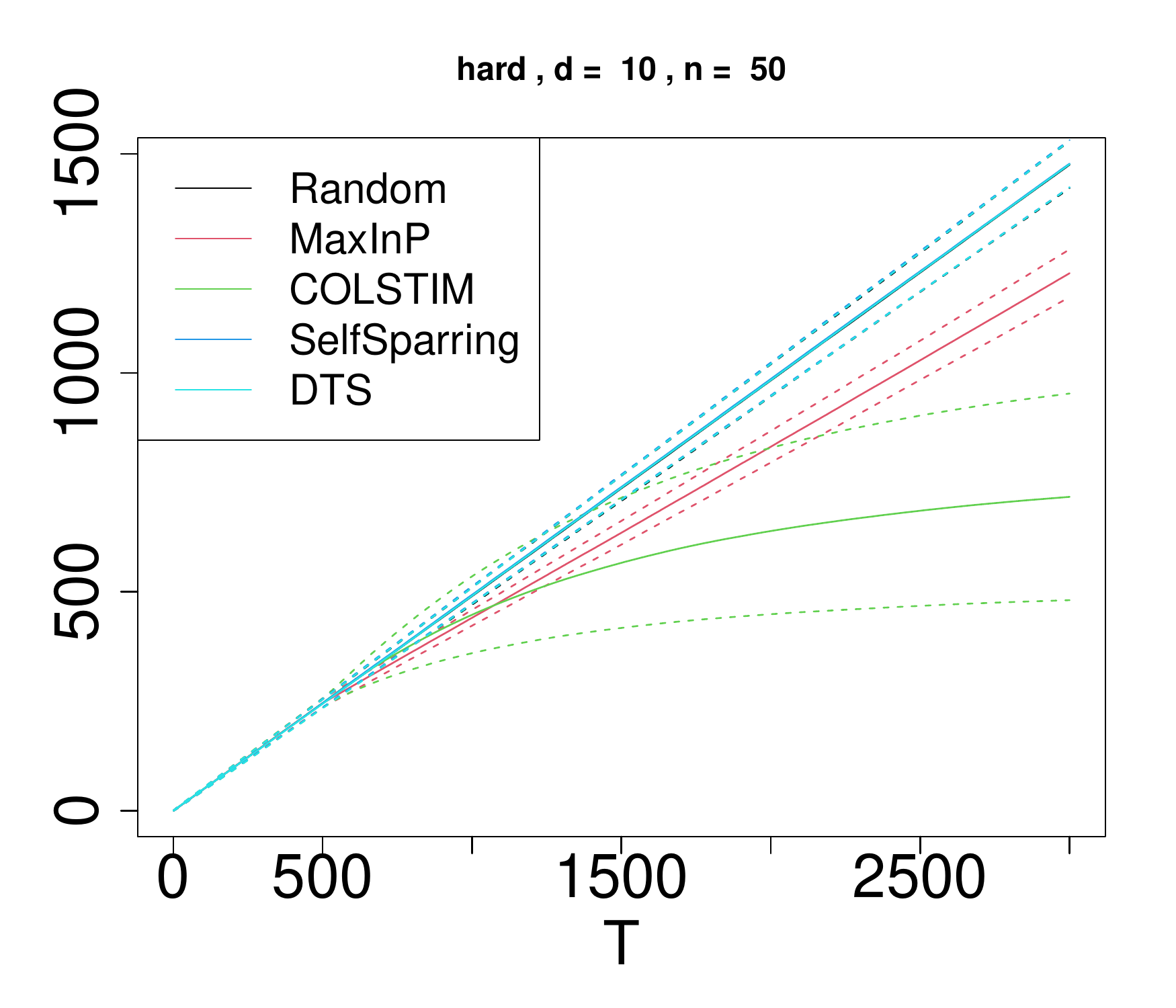}
				\end{subfigure}
				
				\includegraphics[width=.32\linewidth]{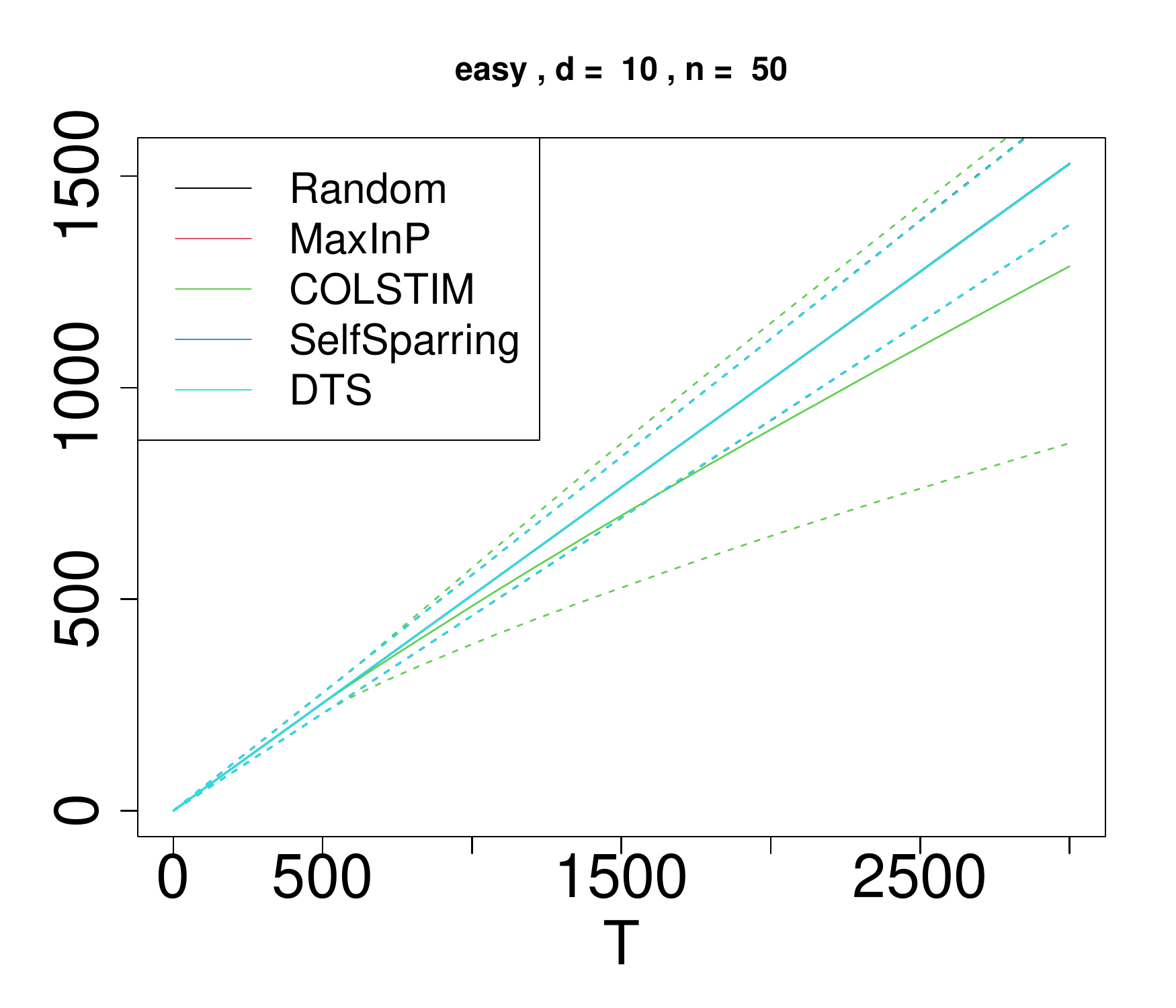}
				\includegraphics[width=.32\linewidth]{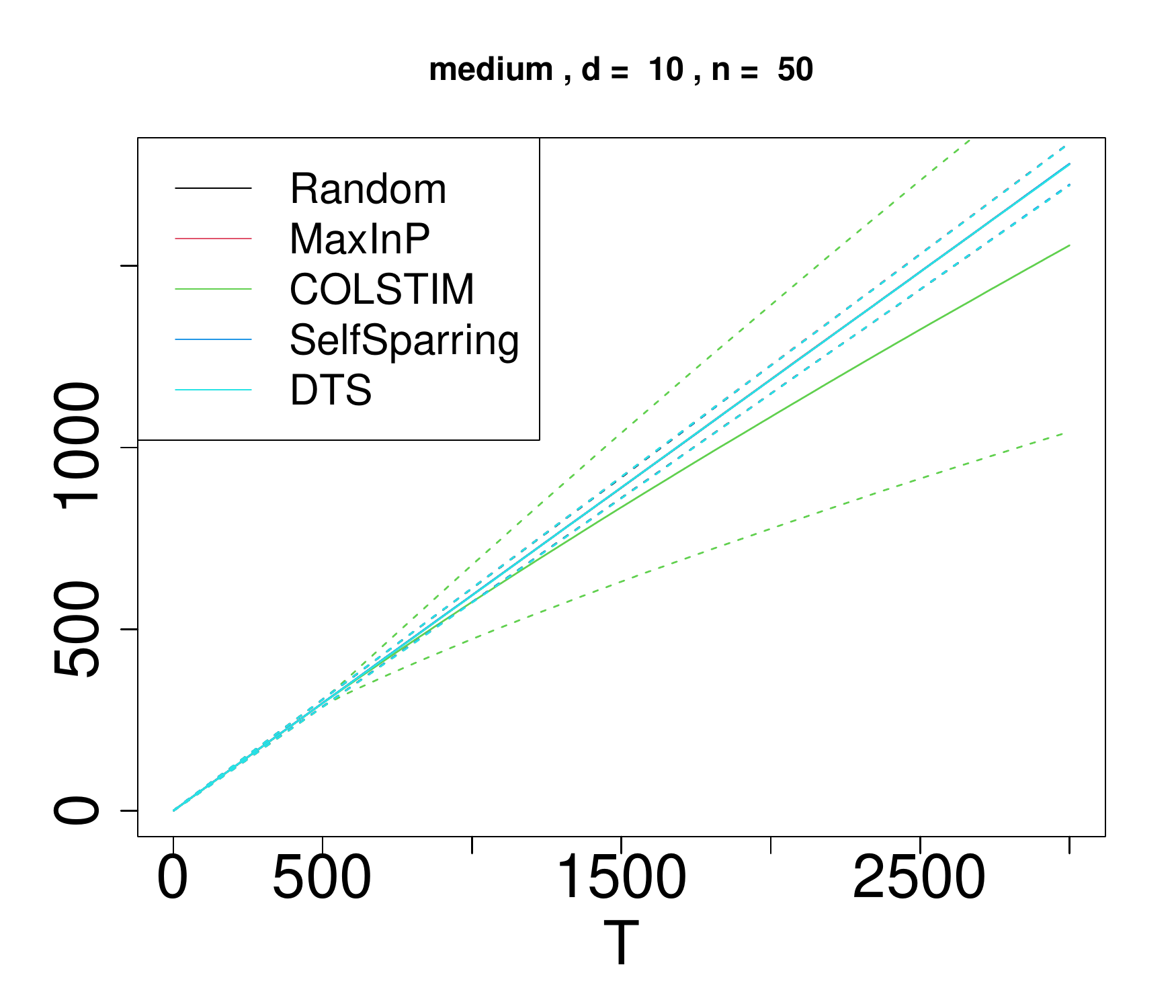}
				\includegraphics[width=.32\linewidth]{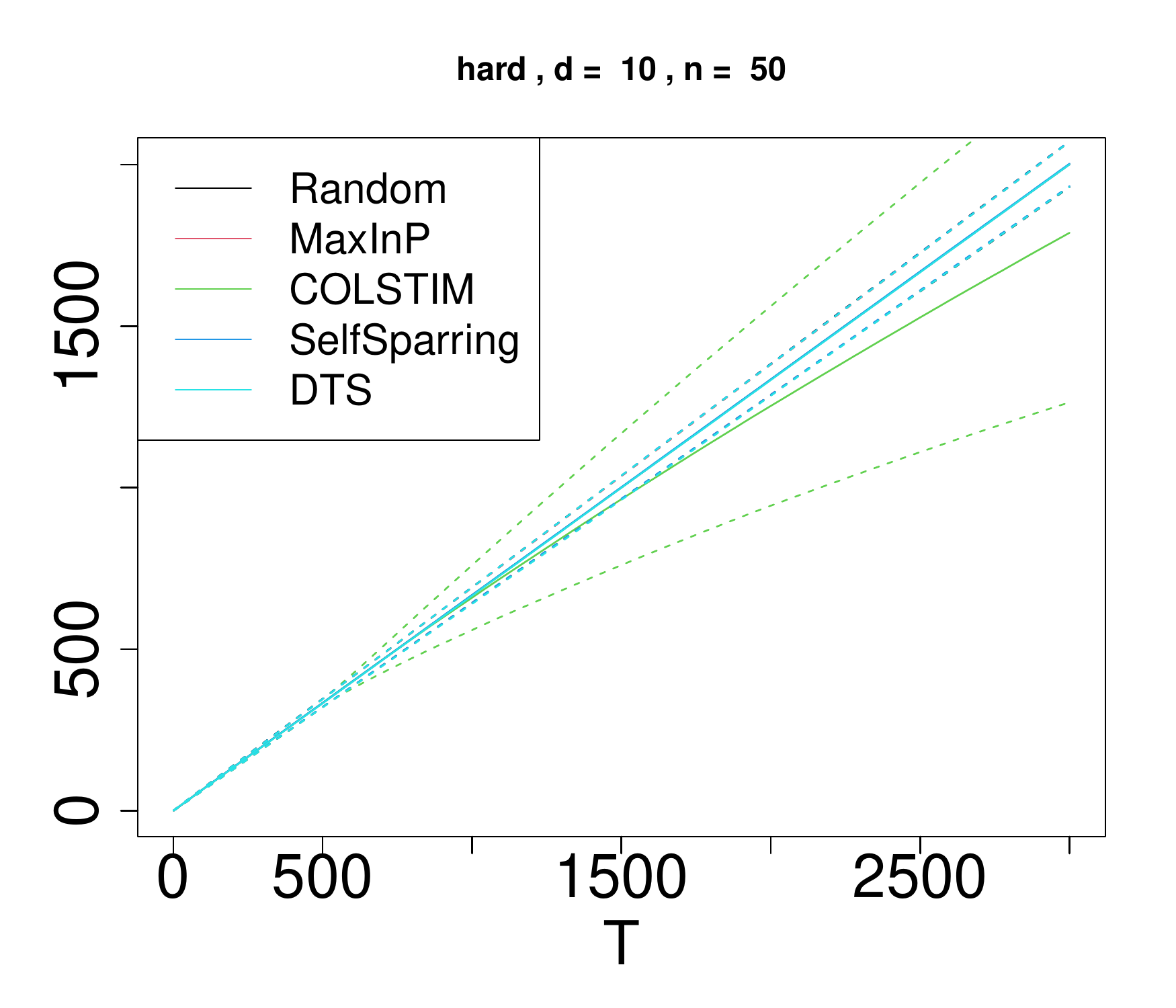}
				\vskip -0.2in
				\caption{Top panel: Averaged cumulative regret of the different methods on $E(d,n,G^*)$ (left), $M(d,n,G^*)$ (middle) and $H(d,n,G^*)$ (right) for $G^*$ being standard Gumbel. Bottom panel: Averaged cumulative regret of the different methods on $E(d,n,G^*)$ (left), $M(d,n,G^*)$ (middle) and $H(d,n,G^*)$ (right) for $G^*$ being the normal distribution with zero and variance $5.$ Here, \Algo{CoLSTIM} is used with $G$ being the standard normal distribution (misspecification).}
				\label{fig:contextual_exp_misspec}
			\end{center}
			\vskip -0.2in
		\end{figure}

		\subsection{Misspecification}
		Given the results of the previous section, one may wonder how sensitive our \Algo{CoLSTIM} method is to misspecification of the perturbation distribution $G,$ i.e., how it performs when $G\neq G^*?$
		To this end, we repeat the experiments for the easy, medium and hard problem instances with $G^*$ being the standard Gumbel distribution, while the \Algo{CoLSTIM} method is used with $G$ being the standard normal distribution.
		%
		Note that MaxInP is designed for exactly this case, in which the perturbation distribution of the underlying CoLST feedback model is the standard Gumbel distribution, so that MaxInP is in a favorable position.
		The top panel in Figure \ref{fig:contextual_exp_misspec} illustrates the result for this setting, which shows a similar picture as Figure \ref{fig:contextual_exp}, where no misspecification is present.

		Finally, we also consider the case where there is an incorrect specification of the parameters in the correct parametric distribution family. 
		More specifically, the bottom panel shows the result of the experiments for the easy, medium and hard problem instances with $G^*$ being the normal distribution with zero mean and variance $5$ and \Algo{CoLSTIM} still uses the standard normal distribution for $G.$ 
        Although our algorithm now performs slightly worse and also has a higher variance, it still performs better than the other algorithms as they still do not show a convergence behavior at all, and our algorithm will probably outperform them eventually for longer time horizons $T.$

		\section{Related Work} \label{sec_related_work}
		So far, a contextualized extension of preference-based bandits with pairwise comparisons (i.e., dueling bandits) has only been studied explicitly by \citet{DuHoShSlZo15} and recently by \citet{saha2020regret}. 
		In the former work, it is assumed that the learner has access to a space of learning policies from which the contextualized von Neumann winner should be found. 
		This learning setting is different from ours, due to the absence of a learning policy space.
		\citet{saha2020regret} assumes a latent linear utility score for each arm as well as a logit link function for modeling the feedback governed by the pairwise preference probabilities. 
		This modeling assumption is a special case of the linear stochastic transitivity models considered in this paper and corresponds to the (contextualized) Bradley-Terry-Luce model. 
		Nevertheless, a lower bound for the regret in this special case of order $\Omega(\sqrt{dT})$ is derived, which is also a lower bound for our more general setting.
		Further, two learning algorithms are suggested and theoretically analyzed, which both essentially use the pair of arms having the highest uncertainty in each round.
		Our proposed learning algorithm, however, uses the idea of randomized learning strategies, which have gained popularity in reward-based bandit problems in recent years (see \citet{vaswani2020old} for a detailed overview).
		This can be attributed to the fact that, unlike the well-known Thompson sampling algorithm, no closed posterior distribution is needed, but only a sampling distribution. 
		
		Finally, it should be mentioned that the utility-based dueling bandits learning scenario with infinitely many arms, where each arm is a $d$-dimensional point, is closely related to the learning scenario considered here.
		In fact, in the latter scenario, if only $K$ many arms, which can vary, are allowed to be selected in each learning round, this corresponds to the contextual learning scenario considered here.
		%
		However, not all learning algorithms are directly amenable to this adaptation \cite{YuJo09}, or if they are, they no longer have theoretical guarantees \cite{AiKaJo14,ku17},  or even did not have them before \cite{gonzalez2017preferential}.

		\section{Conclusion and Future Work}         \label{sec_conclusion}
		We studied the contextualized version of the dueling bandits problem under linear stochastic transitivity models and proposed an algorithm for effective learning in this setting. 
		The algorithm is inspired by randomized learning strategies for numerical bandit problems and has both satisfactory regret bounds as well as excellent numerical performance.
		In addition, we have also shown a lower bound on the weak regret, which is of the same order as for the average regret. 
		This indicates that stronger theoretical guarantees in the worst case sense cannot be obtained for weak regret learners.
		
		For future work, an extension to more general action sets than pairs of arms would be interesting from both a theoretical and a practical point of view as considered by recent works \citep{BrSeCoLi16,SaGo18a,SaGo19b,bengs20,agarwal2020choice}.
		Context-dependent random utility models \cite{train2009discrete} could be a practically meaningful counterpart for linear stochastic transitivity models.
		Another interesting question would be whether one could replace the linearity assumption on the latent utility values with more general assumptions such as a kernel-based correlation as considered by \citet{SuZhBuYu17} for the non-contextual dueling bandits setting. 
		Although our experiments show that our proposed method is robust to misspecification of the perturbation distribution and the corresponding comparison function, one could imagine learning this perturbation distribution online as well. 
		For this purpose, the work by \citet{oliveira2018new} could be relevant, which, however, deals with the batch learning scenario.
		Yet another interesting question would be to investigate whether Assumption \textbf{(A1)} can be relaxed to an assumption like in \cite{li2017provably}, which is also used by MaxInP \citep{saha2020regret}.
		Last but not least, it would certainly be worthwhile to investigate the performance of our suggested algorithm in real-world applications, e.g., in realtime algorithm configuration or online learning-to-rank problems for which preference-based bandit algorithms have been used before \citep{BrSeCoLi16,ScOoWhDe16,OoScDe16,ZhKi16,ScHu18,el2020pool}.   
		
		\subsubsection*{Acknowledgments}
        The authors would like to thank the Paderborn Center for Parallel Computation (PC2) for the use of the OCuLUS cluster.

		\bibliography{contextual_db,dueling.bib}
        \bibliographystyle{icml2022}
		
		\clearpage
		\onecolumn
		\appendix

		\section{List of Symbols}
		
		The following table contains a list of symbols that are frequently used in the main paper as well as in the following supplementary material. \\ \medskip
		\small
			\resizebox{0.99\textwidth}{!}{
		\begin{tabular}{ll}
			\hline
			\multicolumn{2}{c}{\textbf{Basics}} \\
			\hline
			$\prec,\succ$ &  preference relation for objects, i.e., $o \succ o'$ (or $o \prec o'$ ) iff object $o$ is (not) preferred over object $o'$\\
			$1_{\lbrack \cdot \rbrack}$ & indicator function \\
			$\N$  & set of natural numbers (without 0), i.e., $\N = \{1,2,3,\dots\}$ \\
			$\R$ & set of real numbers\\
			$[n]$ & the set $\{1,2,\ldots,n\}$ for some $n\in \N$\\
			$\bx,\bz$ & $d$ -- dimensional (column) vectors \\
			$A^\top$ & transpose of a matrix $A\in \R^{d\times d'}$ \\
			$\langle \bx , \by \rangle$ & inner product of two vectors $\bx,\by \in \R^d$ \\
			$\bI_d$ & $d\times d$ identity matrix \\
			$\bzero_d$ &  $d\times d$ zero matrix \\
			$\| \bx \| $ &  the Euclidean norm of a vector $\bx\in \R^d,$ i.e., $\sqrt{\langle \bx ,\bx  \rangle}$ \\
			$\| \bx \|_{A}$ &  weighted norm of a vector $\bx\in \R^d$ for some positive semi-definite matrix $A\in \R^{d \times d},$ i.e., $ \sqrt{\bx^\top A \bx}$ \\
			$\cB_r(p)$ &  $\ell_p$ ball of radius $r\geq 0$,  i.e., $\cB_r(p) = \{\x \mid \|\x\|_p \leq r\}$, where $\|\cdot\|_p$ denotes the standard $\ell_p$-norm \\
			\hline
			\multicolumn{2}{c}{\textbf{Modelling related}} \\
			\hline
			$n$ & number of arms\\
			$\mathcal{A} = [n]$ & set of arms \\
			$T$ & the time horizon \\
			$\cX$ & the context space (subset of $\R^d$) \\
			$\bx_{t,i}$ & ($d$ -- dimensional) context vector related to arm $i$ at time step $t \in [T]$ (element in $\cX$)\\
			$\bX_t$ & ($d \times n$ -- dimensional) context matrix at time step $t$, i.e.,  $\bX_{t}=(\bx_{t,1}  \ldots \bx_{t,n})$ \\
			$\bz_{t,i,j}$ & the contrast vector of arm $i$ and $j$ at time step $t,$ i.e., $\bz_{t,i,j} = \bx_{t,i} - \bx_{t,j}$ \\
			$\bZ_t$ & ($d \times$ $ n \choose 2 $-- dimensional) contrast matrix at time step $t$, i.e., $\bZ_t = ( \bz_{t,1,2} \,  \bz_{t,1,3} \ldots  \bz_{t,1,n}  \, \bz_{t,2,3} \ldots \bz_{t,n-1,n} )$  \\
			$\btheta^*$ &  ground truth weight parameter of the underlying CoLST model (see \eqref{defi_feedback_context}) \\
			$G^*,F^*$ &  ground truth perturbation distribution and corresponding comparison function of the underlying CoLST model (see \eqref{defi_feedback_context}) \\
			$S_t=\{i_t,j_t\}$ & selected pair of arms at time step $t$ \\
			$Y_t$ & (binary) feedback for the selected pair of arms $S_t$, i.e., $Y_t = 1_{\lbrack i_t \succ j_t \rbrack}$ \\ 
			$u^*_{i}$ & mapping from $\R^{d \times n}$ to $\R$ according to $u^*_{i} (\bX) = \bx_{i}^\top \btheta^*,$ where $\bx_{i}$ is the $i$th column of $\bX$ \\
			$i^*(t)=i^*_t$ & optimal arm at time step $t,$ i.e., the arm with the highest context-dependent utility at time $t$ \\
			&	$	i^*(t) = \arg \max\nolimits_{i \in \cA } u_{i}^*(\bX_t) $\\
			$r_t^w(S_t,\bX_{t})$ &  instantaneous weak regret at time step $t$ for selecting $S_t$ i.e., $r_t^w(S_t,\bX_{t}) = u_{i^*(t)}^*(\bX_t) - \max\nolimits_{k \in S_t} u_{k}^*(\bX_t)$   \\
			$r_t^a(S_t,\bX_{t})$ &  instantaneous average regret at time step $t$ for selecting $S_t$ i.e., $r_t^a(S_t,\bX_t)  =  \frac{2 u_{i^*(t)}^*(\bX_t) - u_{i_t}^*(\bX_t)  - u_{j_t}^*(\bX_t)}{2}$   \\
			$R_T^w, R_T^a $ & cumulative weak or average regret till time $T$ for selecting $(S_t)_{t \in [T]}$ (see \eqref{regret_def})\\
%
			%
			\hline
			\multicolumn{2}{c}{\textbf{CoLSTIM related}} \\
			\hline
			$\hat \btheta_t$ & maximum-likelihood estimate of the weight parameter of the underlying \Algo{CoLSTIM} (see \eqref{def_MLE}) \\
			$c_1>0$ & confidence width parameter of \Algo{CoLSTIM} (hyperparameter of \Algo{CoLSTIM}) \\
			$C_{\mathrm{thresh}}>0$ & threshold parameter of \Algo{CoLSTIM} (hyperparameter of \Algo{CoLSTIM}) \\
			$\tau>0$ & pure exploration length of \Algo{CoLSTIM} (hyperparameter of \Algo{CoLSTIM}) \\
			$p_t \in[0,1]$ & coupling probability (hyperparameter of \Algo{CoLSTIM}) \\
			$G,F$ &  perturbation distribution and comparison function used by \Algo{CoLSTIM} (hyperparameter of \Algo{CoLSTIM}) \\
			$\bM_{t}$ & Gram matrix at time step $t,$ i.e., 	$\bM_t =  \sum_{s=1}^{t-1} \bz_{s,i_s,j_s} \bz_{s,i_s,j_s}^\top   $ \\
			%
			%
			$\tilde \epsilon_{t,i}$ & sampled perturbation variable for arm $i$ at time step $t,$ i.e., $\tilde \epsilon_{t,i} \sim G$ (see line 7 of Alg.\ \ref{alg:CoLSTIM}) \\
			$\epsilon_{t,i}$ & trimmed sampled perturbation variable for arm $i$ at time step $t,$ i.e., $\epsilon_{t,i} = \min( C_{\mathrm{thresh}}, \max(  - C_{\mathrm{thresh}} , \tilde \epsilon_{t,i}  )  )$	 \\
			$\hat u_{t,i}(\bX_{t})$ & estimated utility of arm $i$ in time step $t,$ i.e.,
			$\hat u_{t,i}(\bX_{t}) = \bx_{t,i}^\top \hat \btheta_t $ \\
			$\tilde u_{t,i}(\bX_{t}) $ &  perturbed estimated utility of arm $i$ in time step $t,$ i.e., 
			$\tilde u_{t,i}(\bX_{t}) = \hat u_{t,i}(\bX_{t})  + \epsilon_{t,i} \|\bx_{t,i} \|_{\bM_t^{-1}}  $ \\
						\hline
			\multicolumn{2}{c}{\textbf{Sup-CoLSTIM related}} \\
			\hline
			$c_1, C_{\mathrm{thresh}}, \tau,p_t, G,F$ & hyperparameters of  \Algo{Sup-CoLSTIM} (same meaning as for  \Algo{CoLSTIM}) \\
			$s,S$ & stages in \Algo{Sup-CoLSTIM}, maximal number of stages (set to $\lfloor \log_2(T)\rfloor$) \\
			$A_t^{(s)}$ & arms which are ``sufficiently accurately estimated promising arms'' at stage $s$ and time step $t$ \\
			$\Psi^{(0)}$ & all time steps, where the choice has been made at some stage $s \in [S]$ according to \Algo{CoLSTIM}'s choice mechansim \\
			$\Psi^{(s)}$ & all time steps, where the choice has been made at stage $s \in  [S],$ but not according to \Algo{CoLSTIM}'s choice mechansim \\
			$\hat \btheta_t^{(s)}$ & maximum-likelihood estimate of the weight parameter using only observations from the time steps in $\Psi^{(s)}$  \\
			$\bM_{t}^{(s)}$ & Gram matrix at time step $t$ and stage $s$ i.e., 	$\bM_{t}^{(s)} =  \sum_{ l \in \Psi^{(s)}} \bz_{l,i_l,j_l} \bz_{l,i_l,j_l}^\top   $ \\
			$\tilde \epsilon_{t,i},$ $\epsilon_{t,i}$ &  (trimmed) perturbation variables (same meaning as for  \Algo{CoLSTIM})  \\ 
			$\hat u_{t,i}^{(s)}(\bX_{t})$ & estimated utility of arm $i$ in stage $s$ at time $t,$ i.e.,
			$\hat u_{t,i}^{(s)}(\bX_{t}) = \bx_{t,i}^\top \hat \btheta_t^{(s)} $ \\
			$\tilde u_{t,i}^{(s)}(\bX_{t}) $ &  perturbed estimated utility of arm $i$ in stage $s$ at time $t,$ i.e., 
			$\tilde u_{t,i}^{(s)}(\bX_{t}) = \hat u_{t,i}^{(s)}(\bX_{t})  + \epsilon_{t,i} \|\bx_{t,i} \|_{(\bM_{t}^{(s)})^{-1}} 	  $ \\
			\hline
			
		\end{tabular}}
	
%
		\normalsize
		\clearpage
		
		\section{Lower Bound: Proof of Theorem \ref{theorem:lower_bound_weak_regret}} \label{sec:weak_regret_lower_bound_proof}
		
		\renewcommand{\P}{\mathbb{P}}
		
		\textbf{Theorem \ref{theorem:lower_bound_weak_regret}} \textit{
			For any learning algorithm $\cA$ for the contextual dueling bandit problem under linear stochastic transitivity models in dimension $d$, there exists an instance of the problem characterized by a weight vector $\btheta^* \in \cB_1(2)$, context space $\X \subseteq \cB_1(\infty),$  and a comparison function $F:\R \to [0,1]$, such that the the expected weak regret incurred by $\cA$ in any $T > \max(d,16)$ rounds is 
			$$\mathbb{E} \big[ R_T^w((S_t^\cA)_{t\in[T]}) \big]=\Omega\Big(d\sqrt{T}\Big).$$
			%
			%
			Further, if we restrict $\X \subseteq \cB_1(2)$, then 
			$$\mathbb{E} \big[ R_T^w((S_t^\cA)_{t\in[T]}) \big] = \Omega\Big(\sqrt{dT}\Big).$$ 
			%
		}
		
		\begin{proof}
			\textbf{Case 1. $\X \subseteq \cB_1(\infty)$: }
			
			Our proof is based on a similar line of reasoning as that proposed by \citet{LaSz18} (Chapter 24) for the analysis of linear bandit algorithms.
			We assume the context space $\cX$ to be $\{-1,1\}^d \subset \cB_1(\infty)$, and let $\Theta = \big\{-\frac{1}{\sqrt T},\frac{1}{\sqrt T}\big\}^d$ be the set of possible unknown weight vectors $\btheta^* \in \Theta$.
			Note since $d < T$, we have $\|\btheta\|_2 \leq 1$ for any $\btheta \in \Theta$, and hence $\Theta \subset \cB_1(2)$.
			Assume the perturbation distribution $G$ of the underlying CoLST is the standard Gumbel distribution, i.e., $F$ corresponds to the BTL model (see Examples in Section \ref{sec:introdution}).
			Fix any algorithm, and suppose $(\bell_1,\r_1), \ldots (\bell_T,\r_T)$ be the context vectors of the pairs chosen by the algorithm (learner) for $T$ rounds, i.e., $\bell_t:= \x_{t,i_t}$ and $\r_t:= \x_{t,j_t}$ for all $t \in [T]$. Now, for any $t$, let us denote by $\k_t \in \{\bell_t,\r_t\}$ the context vector of the chosen arms with the higher utility, i.e., 
			$$\k_t = \begin{cases}
				\bell_t, & \mbox{if $\bell_t^\top \btheta^*>\r_t^\top \btheta^*$,} \\
				\r_t, & \mbox{else.}
			\end{cases}$$
			Denote by $\sign(x)$ the sign function for any $x\in \R.$
			Let the sequence of context information matrices $(\bX_t)_{t\in[T]}$ be such that there exists exactly one arm $i^*\in[n]$ such that $x_{t,i}^{i^*} = \sign(\theta_i^*),\, \forall i \in [d] \, \forall t\in[T],$ where  $x_{t,i}^{i^*}$ is the $i^{th}$ component of $i^*$'s context vector $\bx_{t,i^*}$ at time $t,$ and $\theta_i^*$ is the $i^{th}$ component of $\btheta^*.$
			Thus, the best arm is $i^*$ for each time step $t.$
			For sake of brevity, let us denote its corresponding context vector by $\bx^*.$
			
			For any $\btheta \in \Theta$, we denote by $\P_{\btheta}$ the measure on (preference) outcomes induced by the interaction of a fixed algorithm $\cA$ and the contextual dueling bandit instance parametrized by $\btheta$ for $F$ being the BTL model. Also, denote by $\expect_{\btheta^*}[\cdot]$ the expected regret of $\cA$ under the problem instance, induced by the model parameter $\btheta^*$.
			
			Writing $k_{ti}$ for the $i^{th}$ component of $\k_t,$ note that the expected cumulative weak regret of the algorithm for $T$ rounds can be written as 
			\begin{align}
				\expect_{\btheta^*} [R_T^w(\cA)] 
				& = \expect_{\btheta^*} \bigg[  \sum_{t = 1}^T \big( \x_*^{\top}\btheta^* - \k_t^{\top}\btheta^*  \big)  \bigg]
				= \expect_{\btheta^*}\bigg[\sum_{t = 1}^{T}\sum_{i = 1}^{d} (\sign(\theta_i^*) - k_{ti})\theta_i^* \bigg] \nonumber\\
				& = \frac{1}{\sqrt T} \sum_{t = 1}^{T}\sum_{i = 1}^{d}\expect_{\btheta^*}\bigg[\sign(\theta_i^*) \neq \sign(k_{ti}) \bigg] 
				= \frac{1}{\sqrt T} \sum_{i = 1}^{d} \expect_{\btheta^*}\bigg[ \sum_{t = 1}^{T}\sign(\theta_i^*) \neq \sign(k_{ti}) \bigg] \nonumber\\
				& \ge \frac{1}{\sqrt T}\frac{T}{2} \sum_{i = 1}^{d} \P_{\btheta^*}\bigg( \sum_{t = 1}^{T}(\sign(\theta_i^*) \neq \sign(k_{ti}) ) \ge \frac{T}{2} \bigg)  \nonumber\\
				&= \frac{\sqrt T}{2} \sum_{i = 1}^{d}\P_{\btheta^*}(i)~,
				\label{eq:lb1}
			\end{align}
			where $\P_{\btheta}(i) := \P_{\btheta}\bigg( \sum_{t = 1}^{T}(\sign(\theta_i) \neq \sign(k_{ti}) ) \ge \frac{T}{2} \bigg)$ for any $\btheta \in \Theta$. So the only task left is to suitably lower bound the quantity $\sum_{i = 1}^{d}\P_{\btheta^*}(i)$. For this we appeal to the Bretagnolle-Huber inequality (see Theorem 14.2 in \cite{LaSz18}), which is used for analyzing the lower bound of linear bandit algorithms as well.
			
			For every $\btheta \in \Theta$, and $i \in [d]$, we denote by $\btheta^i \in \Theta$ such that $\theta^i_j = \theta_j, \, \forall j \in [d]\setminus\{i\}$ and $\theta^i_i = -\theta_i$. 
			Let us define
			$\P^c_{\btheta^{i}}(i) := \P_{\btheta^{i}}\bigg( \sum_{t = 1}^{T}(\sign(\theta^i_i) \neq \sign(k_{ti}) ) \le \frac{T}{2} \bigg)$. But interestingly note that $\P^c_{\btheta^{i}}(i) = \P_{\btheta}(i)$, and now applying Bretagnolle-Huber inequality we have that for any $\btheta \in \Theta$:
			\begin{align*}
				\P_{\btheta}(i) + \P^c_{\btheta^{i}}(i) = 2\P_{\btheta}(i)  \ge \frac{1}{2}\exp(-KL(\P_{\btheta}||\P_{\btheta^i}))~,
			\end{align*}
			where $KL(P,Q)$ denotes the KL-divergence between two probability measures $P$ and $Q.$
			Recalling that $\P_{\btheta}$ denotes the measure on (preference) outcomes induced by the contextual dueling bandit instance $\btheta \in \Theta$, we can further use the chain rule of KL-divergence to get 
			\begin{align}
				\label{eq:lb2}
				2\P_{\btheta}(i)  \ge \frac{1}{2}\exp(-KL(\P_{\btheta}||\P_{\btheta^i})) = \frac{1}{2}\exp\bigg(- \sum_{t = 1}^T KL\Big(\P_{\btheta}(Y_t)||\P_{\btheta^i}(Y_t)\Big)  \bigg),
			\end{align}
			Now, since $Y_t\sim \mathrm{Ber}\big(F(\z_t^\top\btheta)\big)$ with $\z_{t} = (\bell_t - \r_t)$, we have that $KL\big(\P_{\btheta}(Y_t)||\P_{\btheta^i}(Y_t)\big)$ is the KL-divergence between two Bernoulli distributions.
			We bound this by using the following upper bound: 
			\begin{align} \label{ineq_KL_div}
				KL(\p||\q) \le \sum_{y \in \cY}\frac{p^2(y)}{q(y)} -1,
			\end{align}
			where $\p$ and $\q$ are two probability mass functions on some finite set $\cY$ \citep[][Theorem 1.4]{klub16}.
			
			Recall that we assume a BTL model, so that $G$ is the standard Gumbel distribution and consequently $F$ is the cumulative distribution function of the the standard logistic distribution, so that 
			$$ \P_{\btheta}(Y_t) = \mathrm{Ber}\bigg( \frac{e^{\bell_t^\top\btheta}}{e^{\bell_t^\top\btheta} + e^{\r_t^\top\btheta} } \bigg), \qquad \forall t \in [T]. $$ 
			For sake of brevity, let us denote $\btheta_{\bell} = e^{\bell_t^\top\btheta}$ and $\btheta_{\r} = e^{\r_t^\top\btheta}$; similarly $\btheta^i_{\bell} = e^{\bell_t^\top\btheta^{i}}$ and $\btheta^i_{\r} = e^{\r_t^\top\btheta^{i}}$. Further note that $\btheta^i_{\bell} = \btheta_{\bell}e^{-\frac{2}{\sqrt T}\sign(\btheta_i)\ell_{ti}}$, and $\btheta^i_{\r} = \btheta_{\r}e^{-\frac{2}{\sqrt T}\sign(\btheta_i)r_{ti}}$. Further, abbreviating $\phi := -\frac{2}{\sqrt T}\sign(\theta_i)\ell_{ti}$ and $\psi := -\frac{2}{\sqrt T}\sign(\theta_i)r_{ti}$ we finally get:
			\begin{align}
				KL\Big(\P_{\btheta}(Y_t)||\P_{\btheta^i}(Y_t)\Big) & = KL\Bigg( \mathrm{Ber}\bigg( \frac{\btheta_{\bell}}{\btheta_{\bell} + \btheta_{\r}} \bigg) || \mathrm{Ber}\bigg( \frac{\btheta^i_{\bell}}{\btheta^i_{\bell} + \btheta^i_{\r}} \bigg) \Bigg)\nonumber\\
				& \stackrel{\eqref{ineq_KL_div}}{\leq}  \frac{\bigg(1 - \frac{\btheta_{\bell}}{\btheta_{\bell}+\btheta_{\r}}\bigg)^2}{\bigg(1-\frac{ e^{\phi}\btheta_{\bell}}{ e^{\phi}\btheta_{\bell}+ e^{\psi}\btheta_{\r}}\bigg)}
				+ \frac{\bigg(\frac{\btheta_{\bell}}{\btheta_{\bell}+\btheta_{\r}}\bigg)^2}{\frac{ e^{\phi}\btheta_{\bell}}{ e^{\phi}\btheta_{\bell}+ e^{\psi}\btheta_{\r}}} - 1 \nonumber\\
				& = \frac{\btheta_{\r}^2 (e^\phi\btheta_{\bell}+e^\psi\btheta_{\r})}{(\btheta_{\bell}+\btheta_{\r})^2 e^{\psi}\btheta_{\r}} + \frac{\btheta_{\bell}^2 (e^\phi\btheta_{\bell}+e^\psi\btheta_{\r})}{(\btheta_{\bell}+\btheta_{\r})^2 e^{\phi}\btheta_{\bell}} - 1 \nonumber\\
				& = \frac{(e^{\phi}\btheta_{\bell}+e^\psi\btheta_{\r})(\btheta_{\r}^2 e^{\phi}\btheta_{\bell} + \btheta_{\bell}^2 e^{\psi}\btheta_{\r})}{(\btheta_{\bell}+\btheta_{\r})^2 e^{\psi}\btheta_{\r} e^{\phi}\btheta_{\bell}} - 1 \nonumber\\
				& = \frac{e^\phi\btheta_{\bell} + e^{\psi}\btheta_{\r}}{(\btheta_{\bell} + \btheta_{\r})^2} \frac{e^\psi\btheta_{\bell}+e^\phi\btheta_{\r}}{e^{\phi}e^\psi} - 1 \nonumber\\
				& = \frac{ \btheta_{\bell}\btheta_{\r} }{ (\btheta_{\bell} + \btheta_{\r})^2} \frac{\big( e^{\phi} - e^{\psi} \big)^2}{e^{\phi}   e^{\psi}} \nonumber\\
				& \le \frac{1}{4}\cdot\frac{48}{T} = \frac{12}{T}, \label{eq:kl_bound_1}
			\end{align}
			
			where the last inequality \eqref{eq:kl_bound_1} follows from the fact that 
			$
			\frac{\btheta_{\bell}\btheta_{\r}}{(\btheta_{\bell} + \btheta_{\r})^2} = \frac{1}{\big( \sqrt{\frac{\btheta_{\bell}}{\btheta_{\r}}} + \sqrt{\frac{\btheta_{\r}}{\btheta_{\bell}}} \big)^2} \le \frac{1}{2^2} = \frac{1}{4}
			$, and the second term $\frac{\big( e^{\phi} - e^{\psi} \big)^2}{ e^{\phi} e^{\psi}}$ can shown to be upper bounded by $\frac{48}{T}$ for all $T\ge16$ as follows:
			\begin{align*}
				\frac{\big( e^{\phi} - e^{\psi} \big)^2}{ e^{\phi} e^{\psi}}
				&= \frac{(e^{\phi-\psi} - 1)^2}{e^{\phi-\psi}} \le \frac{(e^{4/\sqrt{T}}-1)^2}{e^{4/\sqrt{T}}} & \text{since } \phi-\psi \in \{-\frac{4}{\sqrt{T}},0,\frac{4}{\sqrt{T}}\} \\
				&\le (e^{4/\sqrt{T}}-1)^2 & \text{assuming } (T \ge 16) \implies (e^{4/\sqrt{T}}\ge1)\\
				&\le \left((e-1) \frac{4}{\sqrt{T}}\right)^2 & (\text{using: } e^x-1\le(e-1)x \text{ for } x\in[0,1]) \\
				&= \frac{16(e-1)^2}{T} \le \frac{48}{T}.
			\end{align*}
			Using this in \eqref{eq:lb2} we further get that:
			\begin{align*}
				\P_{\btheta}(i)  \ge \frac{1}{4}\exp\bigg(- T\frac{12}{T}  \bigg) = \frac{\exp(-12)}{4}.
			\end{align*}
			
			Finally, since this holds for any $\btheta \in \Theta$, and $|\Theta| = 2^d$, averaging over all $\btheta \in \Theta$ we get:
			\begin{align*}
				\frac{1}{|\Theta|}\sum_{\btheta \in \Theta}\sum_{i = 1}^d\P_{\btheta}(i) \ge \frac{d \exp(-12)}{4},
			\end{align*}
			which implies that there exists at least one $\btheta \in \Theta$, say $\tilde \btheta$, such that $\sum_{i = 1}^d\P_{\tilde \btheta}(i) > \frac{d\exp(-12)}{4}$.
			The claim now simply follows from \eqref{eq:lb1} with the choice of $\btheta^* = \tilde \btheta$.
			
			\textbf{Case 2. $\X \subseteq \cB_1(2)$: }
			
			For the second part of the proof, it requires us to restrict $\X \subseteq \cB_1(2)$. To achieve this, we use the context space $\cX = \big\{-\frac{1}{\sqrt d},\frac{1}{\sqrt d}\big\}^d.$ Note in this case the expected cumulative weak regret of any algorithm $\cA$ for $T$ rounds becomes
			
			\begin{align*}
				\expect_{\btheta^*} [R_T^w(\cA)]  &= \expect_{\btheta^*} \bigg[  \sum_{t = 1}^T \big( \x_*^{\top}\btheta^* - \k_t^{\top}\btheta^*  \big)  \bigg] = \frac{1}{\sqrt d}\expect_{\btheta^*}\bigg[\sum_{t = 1}^{T}\sum_{i = 1}^{d} (\sign(\theta_i^*) - k_{ti})\theta_i^* \bigg] \geq \frac{\sqrt T}{2 \sqrt d} \sum_{i = 1}^{d}\P_{\btheta^*}(i)~,
			\end{align*}
			which gives the $1/\sqrt d$ scaling in the regret bound. The rest of the proof follows same as analyzed for \textbf{Case 1} before yielding the desired weak-regret  lower bound for this case.
		\end{proof}

		\section{Proofs of Regret Upper Bounds for \Algo{CoLSTIM}} \label{sec_proof_colstim}
		
		For some fixed constants $c_1,c_2>0$ define the MLE concentration event 
		$$ A_{\mathrm{MLE}} = \{	\forall \{i,j\} \subset [n], t>\tau \, : \,	|  \hat u_{t,i}(\bX_{t}) - u_{i}^*(\bX_{t}) + u_{j}^*(\bX_{t}) - \hat u_{t,j}(\bX_{t})  | \leq c_1 \|\bz_{t,i,j} \|_{\bM_t^{-1}} 	\} $$
		and the (time-dependent) perturbed estimated utility concentration event
		$$	A_{\mathrm{conc},t} = \{	\forall \{i,j\} \subset [n] \, : \, 	| \, \tilde u_{t,i}(\bX_{t}) - \tilde u_{t,j}(\bX_{t}) - \hat u_{t,i}(\bX_{t})  + \hat u_{t,j}(\bX_{t})| \leq c_2 \|\bz_{t,i,j} \|_{\bM_t^{-1}} 	\}.	$$
		Moreover, define	the initial concentration event 
		$$A_{\mathrm{init}} = \{	\forall t>\tau \, : \,	|\hat \btheta_t   -  \btheta^*| \leq 1  \ \mbox{and} \  \bM_t\geq \bI_d	\}.$$ 
		%
		%
		For sake of convenience, we will denote the average regret for selecting pairs $(i_t,j_t)$ at time $t$ by
		$$ \Delta_{i_t,j_t} =  \frac{u_{i^*(t)}^*(\bX_{t}) -  u_{i_t}^*(\bX_{t}) + u_{i^*(t)}^*(\bX_{t}) -  u_{j_t}^*(\bX_{t})}{2}.$$
		%
		
		\subsection{Proof of   Theorem \ref{theorem_regret_bound} and Corollary \ref{corollary_colstim}}
		
		We use the following theorem (proven in Section \ref{subsec:theorem_general_regret}) to prove Theorem \ref{theorem_regret_bound} as well as Corollary \ref{corollary_colstim}.
		\begin{theorem} \label{theorem_general_regret}
			Let $(S_t)_{t \in [T]}$ denote the selected pairs of arms by \Algo{CoLSTIM} and let $R_T^{a}(C) = R_T^a((S_t)_{t \in [T]})$ be its corresponding cumulative average regret.
			Assume there exist constants $c_1>0$ and $c_2>0$ such that for some $p_1,p_2,p_{3,t} \in [0,1]$ it holds that $\prob(A_{\mathrm{MLE}}) \geq 1-p_1,$ $\prob(A_{\mathrm{init}}) \geq 1-p_2,$ and for any given $t>\tau$ and for any possible history $\mathcal{H}_{t-1}$ before the start of round $t$, we have  $\prob(A_{\mathrm{conc},t} \given \mathcal{H}_{t-1}) \geq 1- p_{3,t}.$
			Then, for any constant $c_3\geq  \sqrt{  2 \,d   }$ it holds that
			$$	\mathbb{E}  \left[ R_T^{a}(\Algo{CoLSTIM}) \right] \leq   \Delta_{max} \tau + \Delta_{max} \sum_{t=\tau+1}^T p_{3,t} +  \Delta_{max}(p_1+ p_2 ) T +   \frac{c_3}{2} \Big(	3 \, c_1 + c_2	\Big) \sqrt{T \log\left( \frac{2T}{d} \right) }, $$
			where $\Delta_{max} =   2\sqrt{2}.$
			Moreover, if $\sum_{t=\tau+1}^T \lambda_{\min}^{-1/2}(\bM_t)\leq c \sqrt{T},$ where $c$ is some positive constant and $ \lambda_{\min}(A) $ denotes the smallest eigenvalue of a square matrix $A,$ then the previous inequality holds for any constant $c_3\geq  \sqrt{  2   } c.$ 
			%
			%
		\end{theorem}
		%
		%

		
		\begin{proof}[\bf Proof of Theorem \ref{theorem_regret_bound}]
			By the choice of the initial exploration length $\tau$ and Assumption \textbf{(A1)} we can infer that the smallest eigenvalue of the Gram matrix $\bM_t$ is at least $\rho \tau = \max\left(	1, \frac{d \log(T/d)+2\log(T)}{\mu^2}		\right),$ so that following the lines of proof of Theorem 4 in \cite{vaswani2020old} we have that $\prob(A_{\mathrm{init}}) \geq 1-p_2,$ holds for $p_2=1/T.$
			Similarly, by the choice of $c_1$ and following the lines of proof of Theorem 4 in \cite{vaswani2020old} we have that $\prob(A_{\mathrm{MLE}}) \geq 1-p_1$ holds for $p_1=1/T.$
			
			It remains to derive a suitable choice for $p_{3,t}.$
			For this purpose, fix $t>\tau$ and  $ \mathcal{H}_{t-1}$ arbitrary such that we can leave out the conditioning on $ \mathcal{H}_{t-1}$ in the following. 
			With this,
			\begin{align*}
				\prob(A_{\mathrm{conc},t}^\complement )
				&= \prob(  	\exists \{i,j\} \subset [n]: \,	| \, \tilde u_{t,i}(\bX_{t}) - \tilde u_{t,j}(\bX_{t}) - \hat u_{t,i}(\bX_{t})  + \hat u_{t,j}(\bX_{t})| >  c_2 \|\bz_{t,i,j} \|_{\bM_t^{-1}} ) \\
				&= \prob(  	\exists \{i,j\} \subset [n]: \,	| \, \tilde u_{t,i}(\bX_{t}) - \tilde u_{t,j}(\bX_{t}) - \hat u_{t,i}(\bX_{t})  + \hat u_{t,j}(\bX_{t})| >  c_2 \|\bz_{t,i,j} \|_{\bM_t^{-1}} \given B_t=1 )  \prob( B_t=1 ) \\
				&\quad + \prob(  	\exists \{i,j\} \subset [n]: \,	| \, \tilde u_{t,i}(\bX_{t}) - \tilde u_{t,j}(\bX_{t}) - \hat u_{t,i}(\bX_{t})  + \hat u_{t,j}(\bX_{t})| >  c_2 \|\bz_{t,i,j} \|_{\bM_t^{-1}} \given B_t=0 ) \prob( B_t=0 ) \\
				&\leq p_t + \prob(  	\exists \{i,j\} \subset [n]: \,	| \, \tilde u_{t,i}(\bX_{t}) - \tilde u_{t,j}(\bX_{t}) - \hat u_{t,i}(\bX_{t})  + \hat u_{t,j}(\bX_{t})| >  c_2 \|\bz_{t,i,j} \|_{\bM_t^{-1}} \given B_t=0 )  \\
				&= 	p_t + \prob(  	\exists \{i,j\} \subset [n]: \,	| \, \epsilon_{t,i} \|\bx_{t,i} \|_{\bM_t^{-1}}   -  \epsilon_{t,j}\|\bx_{t,j} \|_{\bM_t^{-1}}   | >  c_2 \|\bz_{t,i,j} \|_{\bM_t^{-1}} \given B_t=0 )  \\
				&\leq 	p_t + \prob(  	\exists \{i,j\} \subset [n]: \,	| \, \epsilon_{t,i} \|\bx_{t,i} \|_{\bM_t^{-1}}   -  \epsilon_{t,j}\|\bx_{t,j} \|_{\bM_t^{-1}}   | >  c_2 \|\bz_{t,i,j} \|_{\bM_t^{-1}} \given B_t=0 )
			\end{align*}
			Note that conditioned on $B_t=0,$ we have that all perturbation variables $(\epsilon_{t,i})_{i\in[n]}$ are the same, i.e., for each $i\in[n]$ it holds that $\epsilon_{t,i} = \epsilon= \min( C_{\mathrm{thresh}}, \max(  - C_{\mathrm{thresh}} , \tilde \epsilon  )  )$ for some  $\tilde \epsilon \sim G.$ 
			Thus, 
			\begin{align*}
				\prob(A_{\mathrm{conc},t}^\complement )
				&\leq 	p_t + \prob(  	\exists \{i,j\} \subset [n]: \,	| \, \epsilon \, | \,\cdot \, \big| \|\bx_{t,i} \|_{\bM_t^{-1}}   -   \|\bx_{t,j} \|_{\bM_t^{-1}}   \big| >  c_2 \|\bz_{t,i,j} \|_{\bM_t^{-1}} \given B_t=0 ) \\
				&\leq 	p_t + \prob(  \,	| \, \epsilon \, | >  c_2  \given B_t=0 ) \\
				&= p_t,
			\end{align*}
			where we used that by the reverse triangle inequality it holds that  
			$$ \big| \|\bx_{t,i} \|_{\bM_t^{-1}}   -   \|\bx_{t,j} \|_{\bM_t^{-1}}   \big| \leq  \|\bx_{t,i} - \bx_{t,j} \|_{\bM_t^{-1}} = \|\bz_{t,i,j} \|_{\bM_t^{-1}} $$
			for any $i,j\in [n]$ and that $| \, \epsilon \, | \leq C_{\mathrm{thresh}} < c_2.$
			As a consequence, we can set 
			$$p_{3,t}=p_t = \min\left(1, \, \frac{\sqrt{2d}}{2\sqrt{t - \tau}} \Big(	3 \, c_1 + c_2	\Big) \sqrt{\log\left( \frac{2T}{d} \right) } \, \right)$$
			such that $\prob(A_{\mathrm{conc},t} \given \mathcal{H}_{t-1}) \geq 1- p_{3,t}$ and 
			\begin{align} \label{sum_p3t_asympt}
				\sum_{t=\tau+1}^T p_{3,t}  \leq  \frac{\sqrt{2d}}{2 } \Big(	3 \, c_1 + c_2	\Big) \sqrt{\log\left( \frac{2T}{d} \right) }  \sum_{t=1}^T \frac{1}{\sqrt{t}} \leq   \frac{\sqrt{2d}}{2 } \Big(	3 \, c_1 + c_2	\Big) \sqrt{T \log\left( \frac{2T}{d} \right) } = O(d\sqrt{T}\log(T)).   
			\end{align}
			Thus, using the fact that with the choices of $c_1$ and $\tau$ we can use $p_1=p_2=1/T,$ while by the choice of $c_2,$ $C_{\mathrm{thresh}}$ and $p_t,$ we can set $p_{3,t}=p_t,$ we obtain the claim by virtue of Theorem \ref{theorem_general_regret}, by using \eqref{sum_p3t_asympt} as well as $c_2 \leq c_1  = O(\sqrt{d\log(T)} )$ and $\tau=o(\sqrt{dT}).$ 
		\end{proof}
		
		%

		\begin{proof}[\bf Proof of 	Corollary \ref{corollary_colstim}]
			Use the same choices of $c_1,c_2$ and $p_1,p_2,p_{3,t}\in [0,1]$ as in the proof of Theorem \ref{theorem_regret_bound}, but use the second statement of Theorem \ref{theorem_general_regret}.
		\end{proof}

		%

		\subsection{Technical Lemmas}

		The following lemma is an adaptation of Lemma 11 from \cite{AbPaSz11} for our setting.
		\begin{lemma} \label{lemma_elliptical_potential}
			Let $\bz_1,\bz_2,\ldots,\bz_t \in \R^d$ be such that $\max_{s\in[t]}\| \bz_s \|^2 \leq 2$ and $\bM_{t+1} = \sum_{s= 1}^{t}  \bz_s \bz_s^\top.$ 
			Further, let $\tau < t$ be such that $ \bM_{\tau +1} \geq I_{d}$ holds, then 
			\begin{align*}
				%
				%
				\sum_{s=\tau+1}^t \sqrt{2} \wedge  \| \bz_s \|_{\bM_s^{-1}}  \leq \sqrt{2\,d\,t  \log\left( \frac{2t}{d}  \right) }.
			\end{align*}
		\end{lemma}
		\begin{proof}
			Note that $x \leq 2\log(1+x)$ holds for any $x\in[0,2],$ so that
			\begin{align} \label{ineq_ell_aux_1}
				\sum_{s=\tau+1}^t (\sqrt{2} \wedge  \| \bz_{s} \|_{\bM_s^{-1}} )^2 
				&\leq \sum_{s=\tau+1}^t \| \bz_{s} \|_{\bM_s^{-1}}^2 
				\leq \sum_{s=\tau+1}^t \, 2\log(1 + \| \bz_{s} \|_{\bM_s^{-1}}^2 ) 
				= 2\log\left( \prod_{s=\tau+1}^t  1 + \| \bz_{s} \|_{\bM_s^{-1}}^2 \right).
				%
			\end{align}
			Denoting by $\det(A)$ the determinant of a square matrix $A,$ we obtain by the matrix determinant lemma 
			\begin{align} \label{ineq_ell_aux_2}
				\begin{split}
					\det(\bM_{t+1}) 
					= \det(\bM_{t}  + \bz_t \bz_t^\top ) 
					&= \det(\bM_{t})  (1+ \bz_t^\top \bM_{t}^{-1} \bz_t) \\
					&= \det(\bM_{\tau+1}) \prod_{s=\tau+1}^t (1+ \bz_s^\top \bM_{s}^{-1} \bz_s) \\
					&= \det(\bM_{\tau+1}) \prod_{s=\tau+1}^t 1+ \| \bz_s^\top \|_{\bM_{s}^{-1}}^2 \\
					&\geq \prod_{s=\tau+1}^t 1+ \| \bz_s^\top \|_{\bM_{s}^{-1}}^2,
				\end{split}
			\end{align}
			where we used in the inequality that $ \bM_{\tau +1} \geq I_{d}$ holds.
			Further, by the determinant-trace inequality, i.e., $\det(A)^{1/d}\leq \frac{\mathrm{tr}(A)}{d}$ for any positive definite matrix $A\in \R^{d\times d},$ we obtain 
			\begin{align} \label{ineq_ell_aux_3}
				d \, \det(\bM_{t+1})^{1/d}  \leq \mathrm{tr}(\bM_{t+1}) = \mathrm{tr}\left(  \sum_{s= 1}^{t}  \bz_s \bz_s^\top  \right) = \sum_{s= 1}^{t}  \mathrm{tr}\left(\bz_s \bz_s^\top \right) = \sum_{s= 1}^{t} \| \bz_s \|^2 \leq 2 t,
			\end{align}
			where $\mathrm{tr}(A)$ denotes the trace of a matrix $A.$
			Combining \eqref{ineq_ell_aux_1}, \eqref{ineq_ell_aux_2} and \ref{ineq_ell_aux_3} we obtain
			\begin{align*}
				\sum_{s=\tau+1}^t (\sqrt{2} \wedge  \| \bz_{s} \|_{\bM_s^{-1}} )^2 
				\leq 2\log\left( \det(\bM_{t+1}) \right) 
				\leq 	2 d \log\left( \frac{2t}{d} \right).
			\end{align*}
			Finally, by the Cauchy-Schwarz inequality we obtain
			\begin{align*}
				\sum_{s=\tau+1}^t \sqrt{2} \wedge  \| \bz_s \|_{\bM_s^{-1}} 
				\leq \sqrt{ (t-\tau)   \sum_{s=\tau+1}^t (\sqrt{2} \wedge  \| \bz_s \|_{\bM_s^{-1}} )^2    } 
				\leq \sqrt{2\,d\,t  \log\left( \frac{2t}{d}  \right) }.
			\end{align*}
		\end{proof}
		
		\begin{lemma} \label{lemma:gaps_cond}
			Let $c_1,c_2$ and $p_1,p_2,p_{3,t} \in [0,1]$ be as in Theorem \ref{theorem_general_regret}.
			Then, for any round $t>\tau$ and any history $ \mathcal{H}_{t-1},$ we have 
			$$	\mathbb{E} [ \Delta_{i_t,j_t}  1_{ A_{\mathrm{MLE}}  \cap  A_{\mathrm{init}} } \given \mathcal{H}_{t-1} ] \leq \Delta_{max} p_{3,t} + \frac12 \Big(	3 c_1 + c_2	\Big)  \mathbb{E}[  \| \bz_{t,i_t,j_t} \|_{\bM_t^{-1}}  1_{  A_{\mathrm{init}} }  \given  \mathcal{H}_{t-1} ],	$$
			where $c_1>0$ and $c_2>0.$ 
		\end{lemma}

		\begin{proof}[\bf Proof of Lemma \ref{lemma:gaps_cond}]
			Fix $t$ and  $ \mathcal{H}_{t-1}$ arbitrary such that both events  $ A_{\mathrm{MLE}} $ and $ A_{\mathrm{init}} $ in the indicator function are true.
			In this way, we can leave out the conditioning on $ \mathcal{H}_{t-1}$ in the following.
			We first bound the (conditional) expected value as follows:
			\begin{align*}
				\mathbb{E} [  \Delta_{i_t,j_t}  1_{ A_{\mathrm{MLE}}  \cap  A_{\mathrm{init}} } ] 
				&= 	\mathbb{E} [  \Delta_{i_t,j_t}  1_{ A_{\mathrm{MLE}}  \cap  A_{\mathrm{init}} }  1_{A_{\mathrm{conc},t}} ]  + \mathbb{E} [  \Delta_{i_t,j_t}  1_{ A_{\mathrm{MLE}}  \cap  A_{\mathrm{init}} }  1_{A_{\mathrm{conc},t}^\complement } ] \\
				&\leq \mathbb{E} [  \Delta_{i_t,j_t}  1_{ A_{\mathrm{MLE}}  \cap  A_{\mathrm{init}} }  1_{A_{\mathrm{conc},t}} ]  + 
				\Delta_{max} \prob\big(	A_{\mathrm{conc},t}^\complement	\big) \\
				&\leq \mathbb{E} [  \Delta_{i_t,j_t}  1_{ A_{\mathrm{MLE}}  \cap  A_{\mathrm{init}} }  1_{A_{\mathrm{conc},t}} ]  + 
				\Delta_{max} p_{3,t} .
				%
				%
			\end{align*}
			On the event  $ A_{\mathrm{MLE}} $ it holds that
			\begin{align} \label{ineq_mle_conc}
				\begin{split}
					\langle \bz_{t,i_t^*,i_t} , \btheta^* - \hat{\btheta}_t \rangle
					&= u_{i_t^*}^*(\bX_{t}) - \hat u_{t,i_t^*}(\bX_{t}) + \hat u_{t,i_t}(\bX_{t}) -  u_{i_t}^*(\bX_{t})    \\
					&\leq  |  u_{i_t^*}^*(\bX_{t}) - \hat u_{t,i_t^*}(\bX_{t}) + \hat u_{t,i_t}(\bX_{t}) -  u_{i_t}^*(\bX_{t})  | \\
					&\leq   c_1 \|\bz_{t,i_t^*,i_t} \|_{\bM_t^{-1}}			\end{split}
			\end{align} 
			and similarly $\langle \bz_{t,i_t^*,j_t} , \btheta^* - \hat{\btheta}_t \rangle \leq c_1 \|\bz_{t,i_t^*,j_t} \|_{\bM_t^{-1}}. $
			Further, note that 
			\begin{align} \label{eq_z_change}
				\bz_{t,i_t^*,j_t} = \bx_{t,i_t^*} -  \bx_{t,j_t} = \bx_{t,i_t^*} -  \bx_{t,i_t} +  \bx_{t,i_t} -  \bx_{t,j_t} = \bz_{t,i_t^*,i_t} +   \bz_{t,i_t,j_t}
			\end{align}
			and by definition of $j_t$ it holds for all $i\in[n]$ that
			\begin{align} \label{ineq_jt_choice}
				\langle \bz_{t,i,i_t} ,  \hat{\btheta}_t \rangle +   c_1\, \|\bz_{t,i,i_t} \|_{\bM_t^{-1}} 
				\leq \langle \bz_{t,j_t,i_t} ,  \hat{\btheta}_t \rangle +   c_1\, \|\bz_{t,j_t,i_t} \|_{\bM_t^{-1}}
				= \langle \bz_{t,j_t,i_t} ,  \hat{\btheta}_t \rangle +   c_1\, \|\bz_{t,i_t,j_t} \|_{\bM_t^{-1}}
			\end{align} 
			Finally, by the definition of $i_t$ it holds for all $i\in[n]$ that
			\begin{align}\label{ineq_it_choice}
				\langle \bx_{t,i} ,  \hat{\btheta}_t \rangle	+  \epsilon_{t,i} \|\bx_{t,i} \|_{\bM_t^{-1}} - \big(	\langle \bx_{t,i_t} ,  \hat{\btheta}_t \rangle + \epsilon_{t,i_t} \|\bx_{t,i_t} \|_{\bM_t^{-1}}\big)
				= \langle \bz_{t,i,i_t} ,  \hat{\btheta}_t \rangle	+  \epsilon_{t,i} \|\bx_{t,i} \|_{\bM_t^{-1}} - \epsilon_{t,i_t} \|\bx_{t,i_t} \|_{\bM_t^{-1}} \leq 0 \, .
			\end{align}
			With these considerations,
			\begin{align*}
				2 \Delta_{i_t,j_t} 
				& = u_{i^*(t)}^*(\bX_{t}) -  u_{i_t}^*(\bX_{t}) + u_{i^*(t)}^*(\bX_{t}) -  u_{j_t}^*(\bX_{t}) \\
				& = \langle \bz_{t,i_t^*,i_t} , \btheta^* \rangle +  \langle \bz_{t,i_t^*,j_t} , \btheta^* \rangle \\
				& =  \langle \bz_{t,i_t^*,i_t} , \btheta^* - \hat{\btheta}_t \rangle +  \langle \bz_{t,i_t^*,i_t} ,  \hat{\btheta}_t \rangle	+ \langle \bz_{t,i_t^*,j_t} , \btheta^* -   \hat{\btheta}_t  \rangle +  \langle \bz_{t,i_t^*,j_t} ,  \hat{\btheta}_t  \rangle \\
				& \stackrel{\eqref{ineq_mle_conc}}{\leq}  c_1 \big(  \|\bz_{t,i_t^*,i_t} \|_{\bM_t^{-1}} +  \|\bz_{t,i_t^*,j_t} \|_{\bM_t^{-1}}  \big)+  \langle \bz_{t,i_t^*,i_t} ,  \hat{\btheta}_t \rangle + \langle \bz_{t,i_t^*,j_t} ,  \hat{\btheta}_t  \rangle \\
				& \stackrel{\eqref{eq_z_change}}{\leq}  c_1 \big(  2\, \|\bz_{t,i_t^*,i_t} \|_{\bM_t^{-1}} +  \|\bz_{t,i_t,j_t} \|_{\bM_t^{-1}}  \big) +  2 \langle \bz_{t,i_t^*,i_t} ,  \hat{\btheta}_t \rangle + \langle \bz_{t,i_t,j_t} ,  \hat{\btheta}_t  \rangle  \\
				& \stackrel{\eqref{ineq_jt_choice}}{\leq}  3 c_1 \|\bz_{t,i_t,j_t} \|_{\bM_t^{-1}}  +  2 \langle \bz_{t,j_t,i_t} ,  \hat{\btheta}_t \rangle + \langle \bz_{t,i_t,j_t} ,  \hat{\btheta}_t  \rangle \\ 
				& =  \ 3 c_1 \|\bz_{t,i_t,j_t} \|_{\bM_t^{-1}}  +   \langle \bz_{t,j_t,i_t} ,  \hat{\btheta}_t \rangle  \\ 
				& \stackrel{\eqref{ineq_it_choice}}{\leq}  3 c_1  \|\bz_{t,i_t,j_t} \|_{\bM_t^{-1}} +  \epsilon_{t,i_t} \|\bx_{t,i_t} \|_{\bM_t^{-1}} - \epsilon_{t,j_t} \|\bx_{t,j_t} \|_{\bM_t^{-1}}  \\ 
				& \leq  \ 3 c_1  \|\bz_{t,i_t,j_t} \|_{\bM_t^{-1}} + c_2 \|\bz_{t,i_t,j_t} \|_{\bM_t^{-1}} , 
			\end{align*}
			where the last inequality holds, since on $A_{\mathrm{conc},t}$ we have that
			$$	 \epsilon_{t,i_t} \|\bx_{t,i_t} \|_{\bM_t^{-1}} - \epsilon_{t,j_t} \|\bx_{t,j_t} \|_{\bM_t^{-1}} \leq c_2 \|\bz_{t,i_t,j_t} \|_{\bM_t^{-1}}.	$$
			Thus,
			\begin{align*}
				\mathbb{E} [ \Delta_{i_t,j_t}  1_{ A_{\mathrm{MLE}}  \cap  A_{\mathrm{init}} } ] 
				&\leq \frac{1}{2}	\mathbb{E} \left[ \big(3 c_1  \|\bz_{t,i_t,j_t} \|_{\bM_t^{-1}} + c_2  \|\bz_{t,i_t,j_t} \|_{\bM_t^{-1}} \big) 1_{ A_{\mathrm{MLE}}  \cap  A_{\mathrm{init}} } 1_{A_{\mathrm{conc},t}} \right] + 
				\Delta_{max} p_{3,t},
			\end{align*}
			from which we can conclude the claim.
			%
			%
			%
		\end{proof}
		
		\subsection{Proof of Theorem \ref{theorem_general_regret}} \label{subsec:theorem_general_regret}

		\begin{proof}[\bf Proof of Theorem \ref{theorem_general_regret}]
			Let  $\Delta_{max} = 2\sqrt{2}$ then 
			$$ 2 \Delta_{i_t,j_t} = u_{i^*(t)}^*(\bX_{t}) -  u_{i_t}^*(\bX_{t}) + u_{i^*(t)}^*(\bX_{t}) -  u_{j_t}^*(\bX_{t}) \leq\Delta_{max}.$$
			Indeed,  by assumption  $ \max\Big(  \| \bx_{t,i}   \|^2 , \| \bx_{t,j}   \|^2 \Big) \leq 1$ so that $ \| \bz_{t,i,j} \|^2 \leq 2$ holds for all $i,j\in[n]$.
			Further, by assumption $ \| \btheta^* \|\leq 1,$ so that by the Cauchy-Schwarz inequality
			\begin{align*}
				u_{i^*(t)}^*(\bX_{t}) -  u_{i_t}^*(\bX_{t}) + u_{i^*(t)}^*(\bX_{t}) -  u_{j_t}^*(\bX_{t})
				&\leq |u_{i^*(t)}^*(\bX_{t}) -  u_{i_t}^*(\bX_{t}) + u_{i^*(t)}^*(\bX_{t}) -  u_{j_t}^*(\bX_{t})| \\
				&\leq \sqrt{ \| \bz_{t,i_t*,i_t} \|^2 \|\btheta^* \|^2  } + \sqrt{ \| \bz_{t,i_t*,j_t} \|^2 \|\btheta^* \|^2  } \\
				&\leq 2\sqrt{2}.
			\end{align*}
			With this, we get
			\begin{align*}
				\, \mathbb{E}[ R_T^{a}(\Algo{CoLSTIM}) ]\, 
				&=    \sum\limits_{t=1}^{T} \mathbb{E} \left( \Delta_{i_t,j_t} \right) \\
				&\leq  \, \Delta_{max} \tau +  \,  \sum\limits_{t=\tau+1}^{T} \mathbb{E} [ \Delta_{i_t,j_t}]  \\
				&\leq  \, \Delta_{max} \tau +  \, \Delta_{max} T \left(\prob\big( A_{\mathrm{init}}^\complement \big) + \prob\big( A_{\mathrm{MLE}}^\complement \big)\right) +  \, \sum\limits_{t=\tau+1}^{T} \mathbb{E} [  \Delta_{i_t,j_t}  1_{ A_{\mathrm{MLE}}  \cap  A_{\mathrm{init}} } ]  \\
				&=  \, \Delta_{max} \tau +  \, \Delta_{max} T \left(\prob\big( A_{\mathrm{init}}^\complement \big) + \prob\big( A_{\mathrm{MLE}}^\complement \big)\right) +  \, \sum\limits_{t=\tau+1}^{T} \mathbb{E} \left[  \mathbb{E} [  \Delta_{i_t,j_t}  1_{ A_{\mathrm{MLE}}  \cap  A_{\mathrm{init}} } \given \mathcal{H}_{t-1} ]  \right],
			\end{align*}
			where the last equality is due to the tower property of the expected value.
			%
			%
			Using  Lemma \ref{lemma:gaps_cond}, we obtain
			\begin{align*}
				\mathbb{E}[ R_T^{a}(\Algo{CoLSTIM}) ] \, 
				&\leq  \Delta_{max} \tau +   \Delta_{max} \sum_{t=1}^T p_{3,t} + \Delta_{max} T \left(\prob\big( A_{\mathrm{init}}^\complement \big) + \prob\big( A_{\mathrm{MLE}}^\complement \big)\right) \\
				&\quad +  \frac12 \Big(	3 \, c_1 + c_2	\Big)  \sum\nolimits_{t=\tau+1}^{T}  \mathbb{E} \left[\mathbb{E}[  \| \bz_{t,i_t,j_t} \|_{\bM_t^{-1}}  1_{  A_{\mathrm{init}} }  \given  \mathcal{H}_{t-1} ]\right].
				%
			\end{align*}
			On  $ A_{\mathrm{init}} $ it holds for any $i,j\in[n]$ and for each time step $t > \tau$ that 
			$$ \| \bz_{t,i,j} \|_{\bM_t^{-1}}  = \sqrt{ \bz_{t,i,j}^\top \bM_t^{-1} \bz_{t,i,j}   }  \leq  \sqrt{ \bz_{t,i,j}^\top  \bz_{t,i,j}   } \leq  \sqrt{2}.$$
			since $ \| \bz_{t,i,j} \|^2 \leq 2$ holds by assumption (on the context vectors $(\bx_{t,i})_{i\in[n]}$) and $ \bM_t\geq \bI_d$ holds on $ A_{\mathrm{init}}.$
			Thus, we can use Lemma \ref{lemma_elliptical_potential} to obtain
			\begin{align*}
				\mathbb{E}[ R_T^{a}(\Algo{CoLSTIM}) ] \,
				&\leq \Delta_{max} \tau +  \Delta_{max} \sum_{t=\tau+1}^T p_{3,t}  +  \Delta_{max} T \left(\prob\big( A_{\mathrm{init}}^\complement \big) + \prob\big( A_{\mathrm{MLE}}^\complement \big)\right) \\
				&\quad +   \frac{c_3}{2} \Big(	3 \, c_1 + c_2	\Big)   \sqrt{T \log\left( \frac{2T}{d} \right) },
			\end{align*}
			for any  $c_3 \geq \sqrt{  2 \,d \, }.$ 
			This lets us infer the first statement of the theorem, since $\prob\big( A_{\mathrm{init}}^\complement \big) \leq p_2$ and $\prob\big( A_{\mathrm{MLE}}^\complement \big)\leq p_1$ by assumption of the theorem.

			For the second statement note that by the Cauchy-Schwarz inequality 
			\begin{align*}
				\sum_{t=\tau+1}^T  \sqrt{2} \wedge  \| \bz_{t,i_t,j_t} \|_{\bM_t^{-1}} 
				\leq  	\sum_{t=\tau+1}^T  \sqrt{ \bz_{t,i_t,j_t}^\top \bM_t^{-1} \bz_{t,i_t,j_t}}
				\leq  \sum_{t=\tau+1}^T \| \bz_{t,i_t,j_t} \| \|\bM_t^{-1/2} \|
				\leq \sqrt{2}  \sum_{t=\tau+1}^T \lambda_{\min}^{-1/2}(\bM_t),
			\end{align*}
			where we used the fact that the Euclidean matrix norm is consistent with the Euclidean vector norm.
			Thus, using the assumption on the smallest eigenvalue of $\bM_t,$ we obtain
			\begin{align*}
				\mathbb{E}[ R_T^{a}(\Algo{CoLSTIM}) ] \,
				&\leq \Delta_{max} \tau + \Delta_{max} \sum_{t=\tau+1}^T p_{3,t}  + \Delta_{max} T \left(\prob\big( A_{\mathrm{init}}^\complement \big) + \prob\big( A_{\mathrm{MLE}}^\complement \big)\right) \\
				&\quad +  \frac{c_3}{2} \Big(	3 \, c_1 + c_2	\Big) \sqrt{T \log\left( \frac{2T}{d} \right) },
			\end{align*}
			for any  $c_3= \sqrt{  2  } c.$
		\end{proof}

		\section{Proof of Regret Upper Bounds for \Algo{Sup-CoLSTIM}} \label{sec_proofs_sup_colstim}
		
		For some fixed constants $c_1,c_2>0$ define the MLE concentration event 
		$$ A_{\mathrm{MLE}}^{(S)} = \{ \forall t>\tau, \forall s\in [S],	\forall \{i,j\} \subset A_t^{(s)} \, : \,	|  \hat u_{t,i}^{(s)}(\bX_{t}) - u_{i}^*(\bX_{t}) + u_{j}^*(\bX_{t}) - \hat u_{t,j}^{(s)}(\bX_{t})  | \leq c_1 \|\bz_{t,i,j} \|_{\bM_t^{-1}} 	\} $$
		and the time-dependent perturbed estimated utility concentration event
		$$	A_{\mathrm{conc},t}^{(S)} = \{	\forall s\in [S], \, 	\forall \{i,j\} \subset A_t^{(s)}  : \, 	| \, \tilde u_{t,i}(\bX_{t}) - \tilde u_{t,j}(\bX_{t}) - \hat u_{t,i}^{(s)}(\bX_{t})  + \hat u_{t,j}^{(s)}(\bX_{t})| \leq c_2 \|\bz_{t,i,j} \|_{\bM_t^{-1}} 	\}.	$$
		Moreover, define the initial concentration event 
		$$A_{\mathrm{init}}^{(S)} = \{	\forall t>\tau, \,  \forall s\in [S] : \,	|\hat \btheta_t^{(s)}   -  \btheta^*| \leq 1 \ \mbox{and} \ \bM_t^{(s)} \geq \bI_d \}.$$ 
		As in Section \ref{sec_proof_colstim} we denote the average regret for selecting pairs $(i_t,j_t)$ at time $t$ again by
		$$ \Delta_{i_t,j_t} =  \frac{u_{i^*(t)}^*(\bX_{t}) -  u_{i_t}^*(\bX_{t}) + u_{i^*(t)}^*(\bX_{t}) -  u_{j_t}^*(\bX_{t})}{2}.$$
		
		\begin{lemma} \label{lemma_indep_observations}
			For all time steps $t \in [T]$ and any stages $s \in [S],$ given any realization of chosen arm-pairs $(i_t,j_t)_{t \in \Psi^{(s)}},$ the corresponding preference observations $(Y_t)_{t \in \Psi^{(s)}}$ are independent Bernoulli distributed random variables with $Y_t$ having success probability $ F^*( u^*_{i_t}(\bX_{t})  - u^*_{j_t}(\bX_{t})).$
		\end{lemma}
		\begin{proof}
			The proof is similar to Lemma 4 in \cite{li2017provably} or Lemma 14 in \cite{saha2020regret}.
		\end{proof}
	
		\begin{lemma} \label{lemma_avg_regret_bound_sup}
			Consider some time step $t \in [T]\backslash[\tau]$ and suppose that $(i_t,j_t)$ is chosen at stage $s_t \in [S].$ 
			Then, on the event $A_{\mathrm{MLE}}^{(S)}$ it holds that $i_t^* \in A_t^{(s)}$ for all $s\leq s_t.$
			Moreover, on  $ A_{\mathrm{MLE}}^{(S)} \cap A_{\mathrm{conc},t}^{(S)}$ it holds that $\Delta_{i_t,j_t}  \leq   \begin{cases}
					\frac{2}{\sqrt{T}}, & \mbox{if $t \in \Psi^{(0)},$} \\
					\frac{4}{2^{s_t}}, & \mbox{else.} \\
				\end{cases}$  
		\end{lemma}
		
		\begin{proof}
			The proof of the first part (i.e., $i_t^* \in A_t^{(s)} \ \forall s\leq s_t$) is analogous to Lemma 6 in \cite{li2017provably} or Part-1 in Lemma 6 in \cite{saha2020regret}, as $A_{\mathrm{MLE}}^{(S)}$ corresponds to the set $\mathcal{E}_X$ \citep{li2017provably} or $\mathcal{E}$ \citep{saha2020regret}.

			Next, let us write for sake of brevity $s=s_t$ and let us assume that $t \in \Psi^{(0)},$ i.e., it holds that
			\begin{align} 
				\label{ineq_widths}
				w_{i,j}^{(s)}(\bX_t) = c_1 \| \bz_{t,i,j} \|_{(\bM_{t}^{(s)})^{-1}}  \leq 1/\sqrt{T}, \quad  \forall i,j \in A_t^{(s_t)}.
			\end{align}
			On the event  $ A_{\mathrm{MLE}}^{(S)} $ it holds that
			\begin{align} \label{ineq_mle_conc_sup}
				\begin{split}
					\langle \bz_{t,i_t^*,i_t} , \btheta^* - \hat{\btheta}_t^{(s)} \rangle
					&= u_{i_t^*}^*(\bX_{t}) - \hat u_{t,i_t^*}^{(s)}(\bX_{t}) + \hat u_{t,i_t}^{(s)}(\bX_{t}) -  u_{i_t}^*(\bX_{t})    \\
					&\leq  |  u_{i_t^*}^*(\bX_{t}) - \hat u_{t,i_t^*}^{(s)}(\bX_{t}) + \hat u_{t,i_t}^{(s)}(\bX_{t}) -  u_{i_t}^*(\bX_{t})  | \\
					&\leq   c_1 \| \bz_{t,i,j} \|_{(\bM_{t}^{(s)})^{-1}} 			\\
%
%
				\end{split}
			\end{align} 
			and similarly $\langle \bz_{t,i_t^*,j_t} , \btheta^* - \hat{\btheta}_t \rangle \leq c_1 \|\bz_{t,i_t^*,j_t} \|_{\bM_t^{-1}}, $ where we used that $i_t^* \in A_t^{(s)}$ for all $s\leq s_t$ from the first part of the lemma.
			On $A_{\mathrm{conc},t}^{(S)}$ we have that
				\begin{align} \label{ineq_pert_conc_sup}
					 \epsilon_{t,i_t} \|\bx_{t,i_t}  \|_{(\bM_{t}^{(s_t)})^{-1}}  - \epsilon_{t,j_t} \|\bx_{t,j_t}  \|_{(\bM_{t}^{(s_t)})^{-1}}  \leq c_2 \|\bz_{t,i_t,j_t}  \|_{(\bM_{t}^{(s_t)})^{-1}} . 
				\end{align}
			Further, since  $t \in \Psi^{(0)}$ the choice of $j_t$ implies for all $i\in A_t^{(s_t)}$ that
			\begin{align} \label{ineq_jt_choice_sup}
				\langle \bz_{t,i,i_t} ,  \hat{\btheta}_t^{(s_t)} \rangle +   c_1\, \|\bz_{t,i,i_t}  \|_{(\bM_{t}^{(s_t)})^{-1}} 
				%
				%
				\leq \langle \bz_{t,j_t,i_t} ,  \hat{\btheta}_t^{(s_t)} \rangle +   c_1\, \|\bz_{t,i_t,j_t} \|_{(\bM_{t}^{(s_t)})^{-1}} 
				\leq \langle \bz_{t,j_t,i_t} ,  \hat{\btheta}_t^{(s_t)} \rangle +  1/\sqrt{T},
			\end{align} 
			while the choice of $i_t$ implies for all $i\in A_t^{(s_t)}$ that
			\begin{align}\label{ineq_it_choice_sup}
				\begin{split}
					\langle \bx_{t,i} ,  \hat{\btheta}_t^{(s_t)} \rangle	+  \epsilon_{t,i} \|\bx_{t,i}  \|_{(\bM_{t}^{(s_t)})^{-1}}
					&- \big(	\langle \bx_{t,i_t} ,  \hat{\btheta}_t^{(s_t)} \rangle + \epsilon_{t,i_t} \|\bx_{t,i_t}  \|_{(\bM_{t}^{(s_t)})^{-1}} \big) \\
					&\quad = \langle \bz_{t,i,i_t} ,  \hat{\btheta}_t^{(s_t)} \rangle	+  \epsilon_{t,i} \|\bx_{t,i} \|_{(\bM_{t}^{(s_t)})^{-1}} - \epsilon_{t,i_t} \|\bx_{t,i_t}  \|_{(\bM_{t}^{(s_t)})^{-1}} \leq 0 \, .
				\end{split}		
			\end{align}
			With these considerations,
			\begin{align*}
				2 \Delta_{i_t,j_t} 
				& = u_{i^*(t)}^*(\bX_{t}) -  u_{i_t}^*(\bX_{t}) + u_{i^*(t)}^*(\bX_{t}) -  u_{j_t}^*(\bX_{t}) \\
				& = \langle \bz_{t,i_t^*,i_t} , \btheta^* \rangle +  \langle \bz_{t,i_t^*,j_t} , \btheta^* \rangle \\
				& =  \langle \bz_{t,i_t^*,i_t} , \btheta^* - \hat{\btheta}_t^{(s_t)} \rangle 
				+  \langle \bz_{t,i_t^*,i_t} ,  \hat{\btheta}_t^{(s_t)} \rangle	
				+ \langle \bz_{t,i_t^*,j_t} , \btheta^* -   \hat{\btheta}_t^{(s_t)}  \rangle 
				+  \langle \bz_{t,i_t^*,j_t} ,  \hat{\btheta}_t^{(s_t)}  \rangle \\
				& \stackrel{\eqref{ineq_mle_conc_sup}}{\leq}  c_1 \big(  \|\bz_{t,i_t^*,i_t}  \|_{(\bM_{t}^{(s)})^{-1}}  +  \|\bz_{t,i_t^*,j_t}  \|_{(\bM_{t}^{(s)})^{-1}}   \big) 
				+  \langle \bz_{t,i_t^*,i_t} ,  \hat{\btheta}_t^{(s_t)} \rangle 
				+ \langle \bz_{t,i_t^*,j_t} ,  \hat{\btheta}_t^{(s_t)}  \rangle \\
				& \stackrel{\eqref{eq_z_change}}{\leq}  c_1 \big(  2\, \|\bz_{t,i_t^*,i_t}  \|_{(\bM_{t}^{(s)})^{-1}}  +  \|\bz_{t,i_t,j_t}  \|_{(\bM_{t}^{(s)})^{-1}}   \big) 
				+  2 \langle \bz_{t,i_t^*,i_t} ,  \hat{\btheta}_t^{(s_t)} \rangle + \langle \bz_{t,i_t,j_t} ,  \hat{\btheta}_t^{(s_t)}  \rangle  \\
				& \stackrel{\eqref{ineq_jt_choice_sup}}{\leq}  3 c_1 \|\bz_{t,i_t,j_t}  \|_{(\bM_{t}^{(s)})^{-1}}   +  2 \langle \bz_{t,j_t,i_t} ,  \hat{\btheta}_t^{(s_t)} \rangle + \langle \bz_{t,i_t,j_t} ,  \hat{\btheta}_t^{(s_t)}  \rangle \\ 
				& =  \ 3 c_1 \|\bz_{t,i_t,j_t}  \|_{(\bM_{t}^{(s)})^{-1}}   +   \langle \bz_{t,j_t,i_t} ,  \hat{\btheta}_t^{(s_t)} \rangle  \\ 
				& \stackrel{\eqref{ineq_it_choice_sup}}{\leq}  3 c_1  \|\bz_{t,i_t,j_t}  \|_{(\bM_{t}^{(s)})^{-1}}  +  \epsilon_{t,i_t} \|\bx_{t,i_t}  \|_{(\bM_{t}^{(s)})^{-1}}  - \epsilon_{t,j_t} \|\bx_{t,j_t}  \|_{(\bM_{t}^{(s)})^{-1}}   \\ 
				& \stackrel{\eqref{ineq_pert_conc_sup}}{\leq}    \ 3 c_1  \|\bz_{t,i_t,j_t}  \|_{(\bM_{t}^{(s)})^{-1}}  + c_2 \|\bz_{t,i_t,j_t}  \|_{(\bM_{t}^{(s)})^{-1}}  \\
				& \stackrel{\eqref{ineq_widths}}{\leq} \frac{4}{\sqrt{T}},
			\end{align*}
			where we used in the last inequality that by choice $c_2\leq c_1.$ This shows the case $t \in \Psi^{(0)}.$

		Finally, showing  $\Delta_{i_t,j_t}  \leq  \frac{4}{2^{s_t}} $ if $t\notin \Psi^{(0)}$ is analogous to Part-3 of Lemma 6 in \cite{saha2020regret}, as the choice of the pair $(i_t,j_t)$ is the same as for Sta'D in this case\footnote{Our $\Psi^{(0)}$ corresponds to $\phi^c$ in \cite{saha2020regret}.}. Note that we do not need to condition on the event $A_{\mathrm{conc},t}^{(S)}$ for this case.

	\end{proof}

		\begin{lemma} \label{lemma_stage_times_bound}
			On the event $A_{\mathrm{init}}^{(S)}$ it holds for any $s\in[S]$ that $\sqrt{|\Psi^{(s)}|} \leq c_1 2^s \sqrt{2\,d\, \log(4T^2/d)}.$
		\end{lemma}
		\begin{proof}
			The proof is analogous to Lemma 7 in \cite{saha2020regret}, as the event  $A_{\mathrm{init}}^{(S)}$ ensures that the smallest eigenvalue of any stage-dependent Gram matrix $\bM_t^{(s)}$ is at least 1\footnote{The stage-dependent Gram matrix $\bM_t^{(s)}$ is denoted by $V_t^{s}$ in \cite{saha2020regret}.}.
		\end{proof}
	
		\begin{proof}[\bf Proof of Theorem \ref{theorem_regret_bound_Sup_colstim}]
			By the choice of the initial exploration length $\tau = d + \max\{ \nicefrac{d^2 \log(T)}{\mu^2 \rho}  ,d/\rho \} $ and Assumption \textbf{(A1)} we can infer that the smallest eigenvalue of the Gram matrix $\bM_t$ is at least $\rho \tau = \max\left(	1, \frac{d \log(T/d)+2\log(T)}{\mu^2}		\right),$ so that following the lines of proof of Lemma 4 in \cite{saha2020regret} we have that $\prob\big(A_{\mathrm{init}}^{(S)}\big) \geq 1-1/T.$ 
			Moreover, due to the choice of $c_1= \frac{3}{2\mu}\sqrt{2 \log( 3 n T^2)}$ as well as Lemma \ref{lemma_indep_observations} (corresponds to Lemma 14 in \cite{saha2020regret}) it holds that $\prob\big(A_{\mathrm{MLE}}^{(S)}\big) \geq 1-1/T.$
			As in the proof of Theorem \ref{theorem_regret_bound} it is straightforward to show that $\prob(A_{\mathrm{conc},t}^{(S)} \given \mathcal{H}_{t-1} ) \geq 1-p_t$ and 
			\begin{align} \label{sum_p3t_asympt_sup_colst}
				\sum_{t=\tau+1}^T p_{t}  \leq  \frac{\sqrt{2d}}{2} \Big(	3 \, c_1 + c_2	\Big) \sqrt{\log\left( \frac{2T}{d} \right) }  \sum_{t=1}^T \frac{1}{\sqrt{t}} \leq   \frac{\sqrt{2d}}{2 } \Big(	3 \, c_1 + c_2	\Big) \sqrt{T \log\left( \frac{2T}{d} \right) } = O(\sqrt{dT\log(n)}\log(T)).   
			\end{align}
			With this, we get
			\begin{align*}
				\mathbb{E}[ R_T^{a}(\Algo{Sup-CoLSTIM}) ]\, 
				&=    \sum\limits_{t=1}^{\tau} \mathbb{E} [ \Delta_{i_t,j_t}]  + \sum\limits_{t \in \Psi^{(0)}} \mathbb{E}[ \Delta_{i_t,j_t}]  +   \sum\limits_{s=1}^{S} \sum\limits_{t \in \Psi^{(s)}} \mathbb{E}[ \Delta_{i_t,j_t}] \\
				&\leq  \Delta_{max} \tau +  \Delta_{max} T \left(\prob\big( (A_{\mathrm{init}}^{(S)})^\complement \big) + \prob\big( (A_{\mathrm{MLE}}^{(S)})^\complement \big)\right) + 	\sum_{t=\tau+1}^T p_{t}   \\
				&\quad  + \sum\limits_{t \in \Psi^{(0)}} \mathbb{E} \left[  \mathbb{E} [  \Delta_{i_t,j_t}  1_{ A_{\mathrm{init}}^{(S)} \cap  A_{\mathrm{MLE}}^{(S)} \cap A_{\mathrm{conc},t}^{(S)} } \given \mathcal{H}_{t-1} ]  \right]  \\
				&\quad +   \sum\limits_{s=1}^{S} \sum\limits_{t \in \Psi^{(s)}} \mathbb{E} \left[  \mathbb{E} [  \Delta_{i_t,j_t}  1_{ A_{\mathrm{init}}^{(S)} \cap A_{\mathrm{MLE}}^{(S)} \cap A_{\mathrm{conc},t}^{(S)} } \given \mathcal{H}_{t-1} ]  \right] \\
				&\leq   \Delta_{max} \tau +   2 \Delta_{max} +  	\sum_{t=\tau+1}^T p_{t}  +  \frac{2|\Psi^{(0)}|}{\sqrt{T}}  +   4 \sum\limits_{s=1}^{S} \frac{|\Psi^{(s)}|}{2^{s}}   \tag{Lemma \ref{lemma_avg_regret_bound_sup}}\\
				&\leq  \Delta_{max} \tau +   2 \Delta_{max} +  	\sum_{t=\tau+1}^T p_{t}   +  \frac{2|\Psi^{(0)}|}{\sqrt{T}}  +   4  c_1   \sqrt{2\,d\, \log(4T^2/d)} \sum\limits_{s=1}^{S}  \sqrt{|\Psi^{(s)}|}  \tag{Lemma \ref{lemma_stage_times_bound}}  \\
				&\leq  \Delta_{max} \tau +   2 \Delta_{max} +  	\sum_{t=\tau+1}^T p_{t}  +  2\sqrt{T}  +   4  c_1   \sqrt{2\,d\, \log(4T^2/d)} \sqrt{T \log(T)}  \tag{Cauchy-Schwarz inequality and $|\Psi^{(s)}|\leq T$}.
			\end{align*}
			Using \eqref{sum_p3t_asympt_sup_colst} and noting that $c_1=  O(\sqrt{\log( n T)})$ as well as $\tau=o(\sqrt{dT}),$ we can conclude the proof.
		\end{proof}
		
		\clearpage
		
		\section{SGD vs.\ Full Maximum-Likelihood-Estimation} \label{sec:sgd_vs_mle}
		
		Figure \ref{fig:sgdfullmleeasybtl} illustrates the difference in regret for the easy problem scenario $E(10,50,G^*)$ for $G^*$ being the standard Gumbel distribution.
		There is not much of a difference between the resulting regret curves for MaxInP, while the full MLE variant of \Algo{CoLSTIM} performs slightly better than its SGD variant.
		Moreover, the full MLE variant of \Algo{CoLSTIM} has a smaller fluctuation in the sense that its range of the standard error is smaller than of its SGD variant.
		\begin{figure}[ht]
			\centering
			\includegraphics[width=0.7\linewidth]{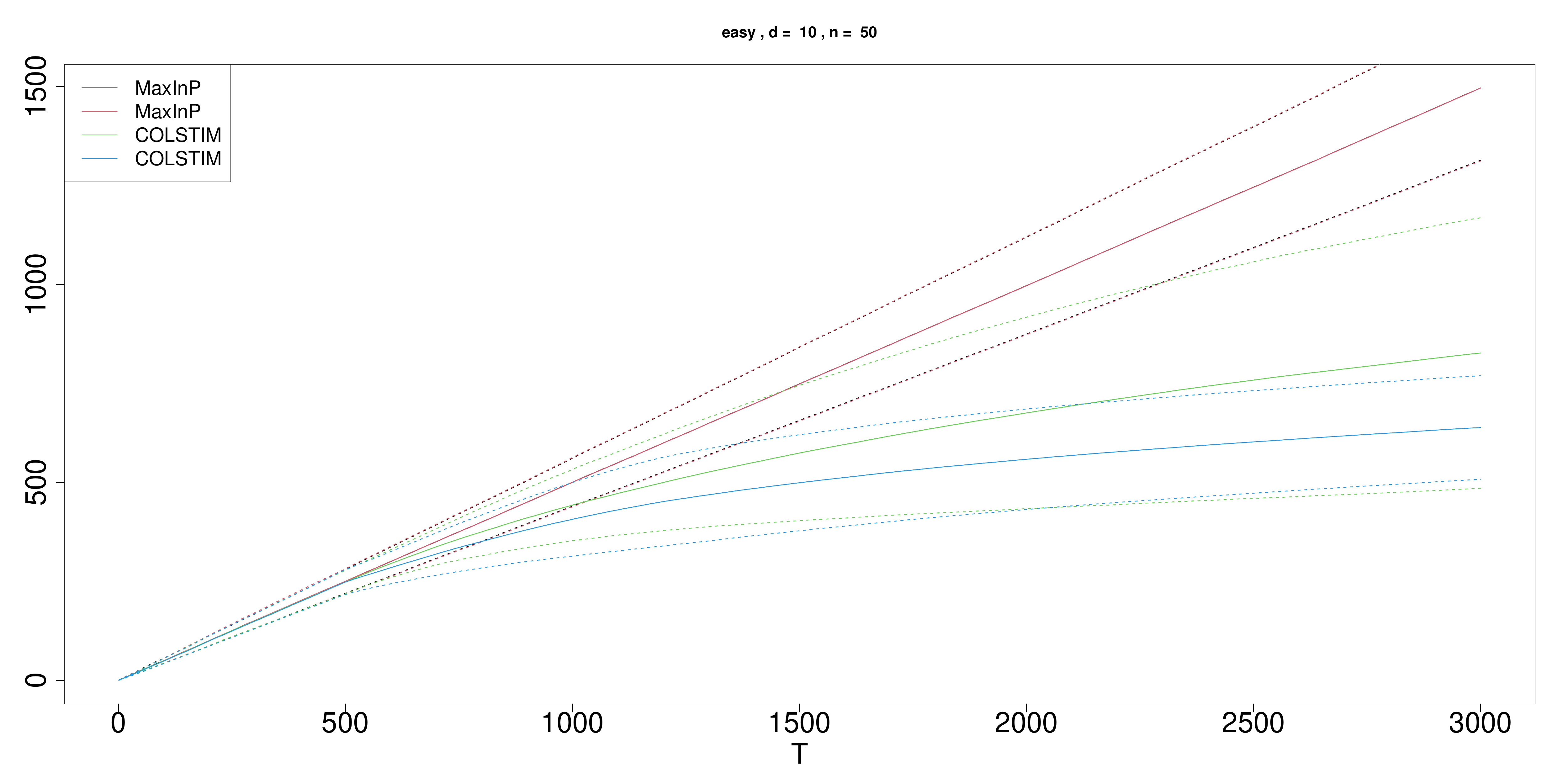}
			\caption{Averaged cumulative regret of the SGD variants of \Algo{CoLSTIM} and MaxInP and their full MLE variants on $E(d,n,G^*)$ for $G^*$ being standard Gumbel.}
			\label{fig:sgdfullmleeasybtl}
		\end{figure}

		However, regarding the average elapsed runtimes (in seconds), we see a huge difference between the SGD variants and the full MLE variants as Table \ref{tab:SGD_MLE} shows.
		
		\begin{table}[ht]
			\centering
			\label{tab:SGD_MLE}
			\caption{Averaged cumulative runtimes and the corresponding standard deviations (in brackets) of the SGD variants of \Algo{CoLSTIM} and MaxInP and their full MLE variants on $E(d,n,G^*)$ for $G^*$ being standard Gumbel.}
			\begin{tabular}{r|r}
				\hline
				& avg. runtimes (std) \\
				\hline
				MaxInP-SGD & 158.00 (4.04) \\ 
				MaxInP-MLE & 910.13 (113.99) \\ 
				\Algo{CoLSTIM-SGD} & 7.01  (0.15) \\ 
				\Algo{CoLSTIM-MLE} & 747.04  (113.63) \\ 
				\hline
			\end{tabular}
		\end{table}

	\end{document}